\documentclass[10pt,twocolumn,letterpaper]{article}

\usepackage{cvpr}              %
\usepackage[accsupp]{axessibility}

\usepackage[dvipsnames]{xcolor}

\definecolor{cvprblue}{rgb}{0.21,0.49,0.74}
\usepackage[pagebackref,breaklinks,colorlinks,citecolor=cvprblue]{hyperref}

\usepackage{url}
\usepackage[utf8]{inputenc}
\usepackage{booktabs} %
\usepackage{amsfonts} %
\usepackage{nicefrac} %
\usepackage{microtype} %
\usepackage{natbib}
\usepackage{gensymb}
\usepackage{color}
\usepackage{caption}
\usepackage{subcaption}
\usepackage{amsthm}
\usepackage{mathrsfs}
\usepackage{amsmath}
\usepackage{amssymb}
\usepackage{mathtools}
\usepackage{wrapfig}
\usepackage{soul}
\usepackage[ruled,longend, algo2e]{algorithm2e}

\usepackage{graphicx}
\usepackage[dvipsnames]{xcolor}
\usepackage{svg}
\usepackage{diagbox}
\usepackage{algorithm}
\usepackage{algorithmic}
\usepackage{hhline}
\usepackage{multirow}
\usepackage{wrapfig,lipsum,booktabs}

\newcommand{\Mat}{\boldsymbol}

\newcommand{\real}{\mathbb{R}}

\DeclareMathOperator{\mean}{\mathbb{E}}
\DeclareMathOperator{\KL}{\mathcal{D}_{KL}}
\DeclareMathOperator{\gauss}{\mathcal{N}}
\DeclareMathOperator{\uniform}{\mathcal{U}}

\theoremstyle{plain}
\newtheorem{theorem}{Theorem}

\newtheorem{lemma}{Lemma}

\theoremstyle{definition}

\theoremstyle{remark}
\newtheorem{remark}{Remark}

\def\paperID{13062} %
\def\confName{CVPR}
\def\confYear{2024}

\title{Taming Mode Collapse in Score Distillation for Text-to-3D Generation}

\author{Peihao Wang\textsuperscript{1}\thanks{Work done during an internship with Meta.}, Dejia Xu\textsuperscript{1}, Zhiwen Fan\textsuperscript{1}, Dilin Wang\textsuperscript{2}, Sreyas Mohan\textsuperscript{2}, Forrest Iandola\textsuperscript{2}, \\
Rakesh Ranjan\textsuperscript{2}, Yilei Li\textsuperscript{2}, Qiang Liu\textsuperscript{1}, Zhangyang Wang\textsuperscript{1}, Vikas Chandra\textsuperscript{2} \\
\textsuperscript{1}The University of Texas at Austin, \textsuperscript{2}Meta Reality Labs\\
{\tt\small\{peihaowang, dejia, zhiwenfan, atlaswang\}@utexas.edu, lqiang@cs.utexas.edu} \\
{\tt\small \{wdilin, sreyasmohan, fni, rakeshr, yileil, vchandra\}@meta.com}\\
\tt\small \href{https://vita-group.github.io/3D-Mode-Collapse/}{vita-group.github.io/3D-Mode-Collapse/}
}

\begin{document}
\maketitle

\begin{abstract}
Despite the remarkable performance of score distillation in text-to-3D generation, such techniques notoriously suffer from view inconsistency issues, also known as ``Janus" artifact, where the generated objects fake each view with multiple front faces.
Although empirically effective methods have approached this problem via score debiasing or prompt engineering, a more rigorous perspective to explain and tackle this problem remains elusive.
In this paper, we reveal that the existing score distillation-based text-to-3D generation frameworks degenerate to maximal likelihood seeking on each view independently and thus suffer from the mode collapse problem, manifesting as the Janus artifact in practice.
To tame mode collapse, we improve score distillation by re-establishing the entropy term in the corresponding variational objective, which is applied to the distribution of rendered images. Maximizing the entropy encourages diversity among different views in generated 3D assets, thereby mitigating the Janus problem.
Based on this new objective, we derive a new update rule for 3D score distillation, dubbed \textit{Entropic Score Distillation (ESD)}.
We theoretically reveal that ESD can be simplified and implemented by just adopting the classifier-free guidance trick upon variational score distillation.
Although embarrassingly straightforward, our extensive experiments demonstrate that ESD can be an effective treatment for Janus artifacts in score distillation.
\end{abstract}

\section{Introduction}

Recent advancements in text-to-3D technology have attracted considerable attention, particularly for its pivotal role in automating high-quality 3D content. This is especially crucial in fields such as virtual reality  and gaming, where 3D content forms the bedrock. While numerous techniques are available, the prevailing text-to-3D approach is based on score distillation~\citep{poole2022dreamfusion}, popularized by DreamFusion and its follow-up works \citep{wang2023score, lin2023magic3d, chen2023fantasia3d, tsalicoglou2023textmesh, metzer2023latent, wang2023prolificdreamer}.

Score distillation leverages a pre-trained 2D diffusion model to sample over the 3D parameter space (\ie Neural Radiance Fields (NeRF)~\citep{mildenhall2020nerf}) such that views rendered from a random angle satisfy the statistics of the image distribution.
This algorithm is implemented by backpropagating the estimated score of each view via the chain rule.
Despite the notable progress achieved with score distillation-based approaches, it is widely observed that 3D content generated using score distillation suffers from the \emph{Janus} problem \citep{hong2023debiasing}, referring to the artifacts that generated 3D objects contain multiple canonical views (see Fig. \ref{fig:teaser}). 

To understand this drawback of score distillation, we draw the theoretical connection between the Janus problem and \emph{mode collapse}, a statistical term describing a distribution concentrating on the high-density area while losing information about the probability tail.
We first uncover that the optimization of existing score distillation-based text-to-3D generation degenerates to a maximum likelihood objective, making it susceptible to model collapse.
As pre-trained diffusion models are biased to frequently encountered views~\citep{hong2023debiasing}\footnote{For example, it is common that a frontal view of a cat is more likely to be sampled from latent diffusion models than the back view.}, this oversight leads all views opt to convergence toward the point with the highest likelihood, manifesting as the Janus artifact in practical applications.
The main limitation of current methods is that their distillation objectives solely maximize the likelihood of each view independently, without considering the diversity between different views.

\begin{figure*}[t]
    \centering
    \includegraphics[width=0.9\linewidth]{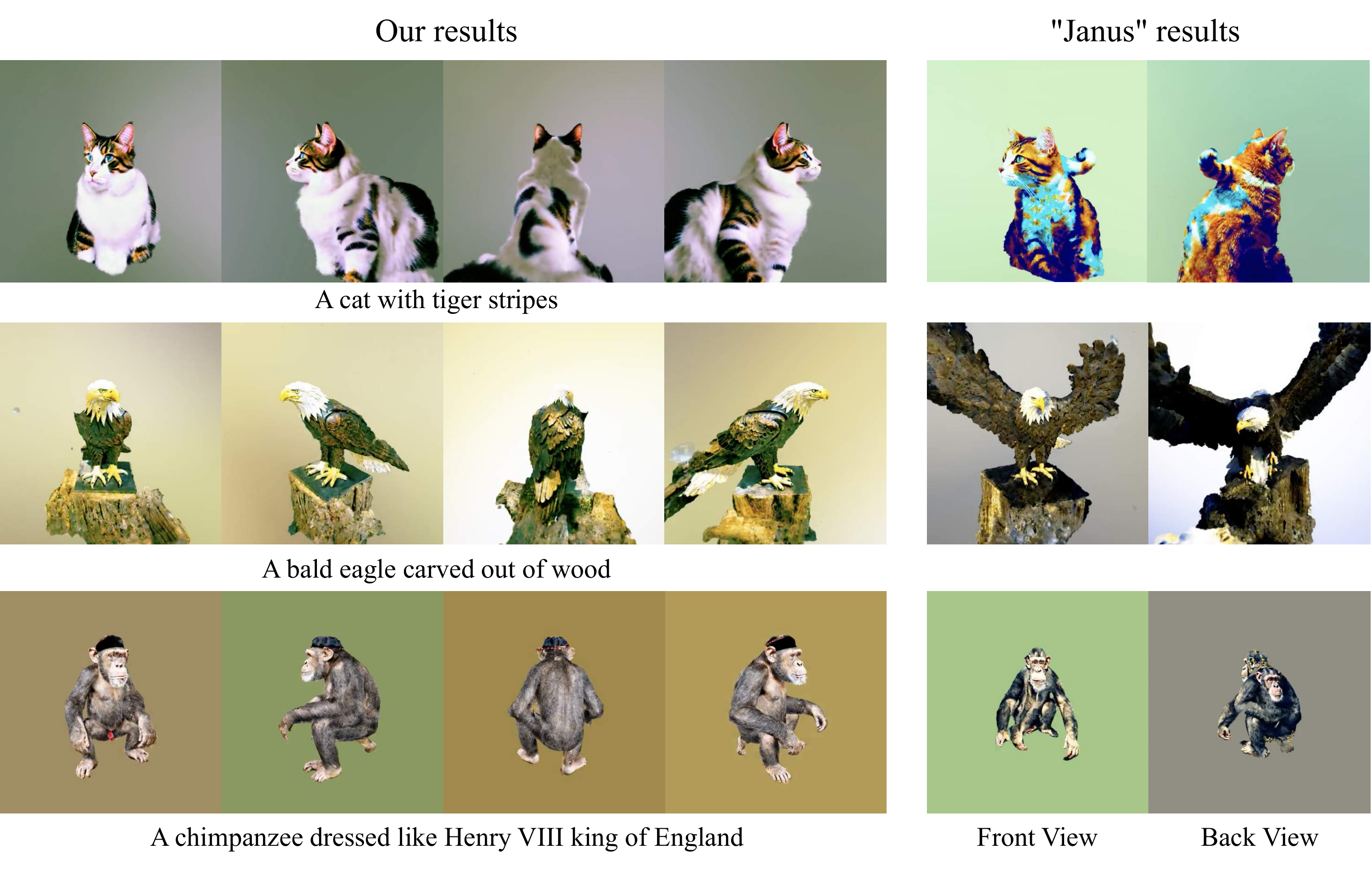}
    \vspace{-1em}
    \caption{\small \textbf{A Preview of Qualitative Results.} We present the front and back views of objects synthesized by VSD (ProlificDreamer) on the right two columns, and four views of our generated results on the left. VSD suffers from ``Janus'' problem, where both front and back views contain a frontal face of the targeted object, while our method effectively mitigates this artifact. Please refer to more results in Appendix \ref{sec:more_vis_res}.}
    \label{fig:teaser}
    \vspace{-1em}
\end{figure*}

To address the aforementioned issue, we propose a principled approach \emph{Entropic Score Distillation} (ESD), which regularizes the score distillation process by entropy maximization of the rendered image distribution, thereby enhancing the diversity of views in generated 3D assets and alleviating the Janus problem. 
Our derived ESD update admits a simple form as a weighted combination of scores for pre-trained image distribution and rendered image distribution. 
Compared with Score Distillation Sampling (SDS)~\citep{poole2022dreamfusion}, our ESD involves the score of the rendered image distribution, serving to maximize the entropy of the rendered image distribution. 
Unlike Variational Score Distillation (VSD)~\citep{wang2023prolificdreamer}, the learned score function of the rendered image distribution does not depend on the camera pose.
This subtle difference has a more profound impact, as we show the score function of rendered images modeled by VSD corresponds to an objective with fixed entropy, thereby having no influence on view variety.
In contrast, ESD optimizes for a Kullback-Leibler divergence with a non-constant entropy term parameterized by the 3D model, leading to an effect that encourages diversity among different views.

In practice, we find it challenging to optimize the score of the rendered image distribution without conditioning on the camera pose. To facilitate training, we discover that the gradient from the entropy can be decomposed into a combination of scores: one depends on the camera pose, and the other independent of it, with a coefficient interacting between these two terms.
Through this theoretical establishment, we are able to adopt a handy implementation of ESD by Classifier Free Guidance (CFG) trick \cite{ho2022classifier} where conditional and unconditional scores are trained alternatively and mixed during inference.

Through extensive experiments with our proposed ESD, we demonstrate its efficacy in alleviating the Janus problem and its significant advantages in improving 3D generation quality when compared to the baseline methods~\citep{poole2022dreamfusion,wang2023prolificdreamer} and other remedy techniques \citep{hong2023debiasing, armandpour2023re}.
As a side contribution, we also borrow two inception scores \cite{salimans2016improved} to evaluate text-to-3D results and numerically probe model collapse in score distillation. We show these two metrics can effectively characterize the quality and diversity of views, highly matching our qualitative observations.

\begin{figure*}[t]
    \centering
    \includegraphics[width=\textwidth]{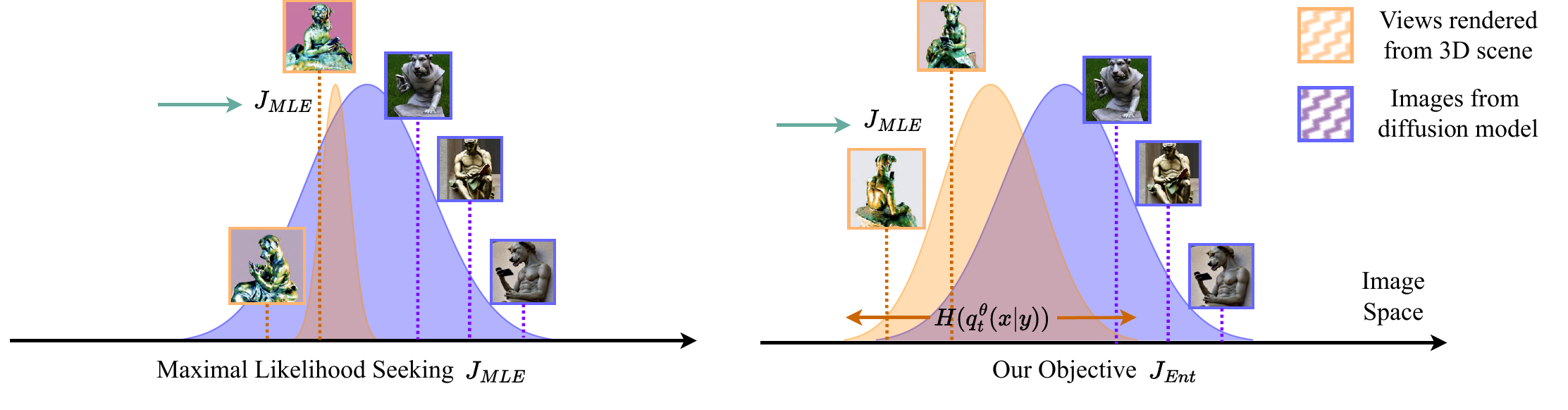}
    \vspace{-2.5em}
    \caption{\small \textbf{Illustration of the effect of entropy regularization.} Learned image distributions often exhibit a higher probability mass for objects' frontal faces. Pure maximal likelihood seeking is opt to mode collapse (Sec. \ref{sec:mode_collapse}). Adding entropy regularization can expand the support of fitted distribution $q^{\Mat{\theta}}_t(\Mat{x} | \Mat{y})$ with mode-covering behavior (Sec. \ref{sec:method}).}
    \label{fig:mle_vs_kl}
    \vspace{-1em}
\end{figure*}

\section{Background}
\label{sec:prelim}

\subsection{Diffusion Models}
Diffusion models, as demonstrated by various works \citep{sohl2015deep, ho2020denoising, song2019generative, song2020score}, have shown to be highly effective in text-to-image generation.
Technically, a diffusion model learns to gradually transform a normal distribution $\gauss(\Mat{0}, \Mat{I})$ to the target distribution $p_{data}(\Mat{x} | \Mat{y})$ where $\Mat{y}$ denotes the text prompt embeddings.
The sampling trajectory is determined by a forward process with the conditional probability $p_t(\Mat{x}_t | \Mat{x}_0) = \gauss(\Mat{x}_t | \alpha_t \Mat{x}_0, \sigma_t^2 \Mat{I})$, where $\Mat{x}_t \in \real^D$ represents the sample at time $t \in [0, T]$, and $\alpha_t, \sigma_t > 0$ are time-dependent diffusion coefficients.
Consequently, the distribution at time $t$ can be formulated as $p_t(\Mat{x}_t | \Mat{y}) = \int p_{data}(\Mat{x}_0 | \Mat{y}) \gauss(\Mat{x}_t | \alpha_t \Mat{x}_0, \sigma_t^2 \Mat{I}) d \Mat{x}_0$.
Diffusion models generate samples through a reverse process starting from Gaussian noises, which can be described by the ODE: $\mathrm{d} \Mat{x}_t / \mathrm{d}t = -\nabla_{\Mat{x}} \log p_t(\Mat{x}_t)$ with the boundary condition $\Mat{x}_T \sim \gauss(\Mat{0}, \Mat{I})$ \citep{song2020score, song2020denoising, liu2023genphys}.
Such a process requires the computation of \textit{score function} $\nabla_{\Mat{x}} \log p_t(\Mat{x}_t)$ which is often obtained by fitting a time-conditioned noise estimator $\Mat{\epsilon_{\phi}}: \real^D \rightarrow \real^D$ using score matching loss \citep{hyvarinen2005estimation, vincent2011connection, song2020sliced}.

\subsection{Text-to-3D Score Distillation}
Score distillation based 3D asset generation requires representing 3D scenes as learnable parameters $\Mat{\theta} \in \real^{N}$ equipped with a differentiable renderer $g(\Mat{\theta}, \Mat{c}): \real^{N} \rightarrow \real^{D}$ that projects 3D scene $\Mat{\theta}$ into images with respect to the camera pose $\Mat{c}$. Here $N, D$ are the dimensions of the 3D parameter space and rendered images, respectively.
Neural radiance fields (NeRF) \citep{mildenhall2020nerf} are often employed as the underlying 3D representation for its capability of modeling complex scenes.

Recent works \citep{poole2022dreamfusion, wang2023score, lin2023magic3d, chen2023fantasia3d, tsalicoglou2023textmesh, metzer2023latent, wang2023prolificdreamer, huang2023dreamtime, wang2023steindreamer} demonstrate the feasibility of using a pretrained 2D diffusion model to guide 3D object creation. 
Below, we elaborate on two score distillation schemes, adopted therein: \emph{Score Distillation Sampling} (SDS)~\citep{poole2022dreamfusion} and \emph{Variational Score Distillation} (VSD)~\citep{wang2023prolificdreamer}.

\paragraph{Score Distillation Sampling (SDS).}
SDS updates the 3D parameter $\Mat{\theta}$ as follows \footnote{Without special specification, expectations are taken over all relevant random variables and Jacobian matrices are transposed by default.}:
\begin{align} \label{eqn:sds}
\nabla_{\Mat{\theta}} J_{SDS}(\Mat{\theta}) = -\mean \left[ \omega(t) \frac{\partial g(\Mat{\theta}, \Mat{c})}{\partial \Mat{\theta}} \left(\sigma_t \nabla \log p_t(\Mat{x}_t | \Mat{y}) - \Mat{\epsilon}\right) \right],
\end{align}
where the expectation is taken over timestep $t \sim \uniform[0, T]$, Gaussian noises $\Mat{\epsilon} \sim \gauss(\Mat{0}, \Mat{I})$, and camera pose $\Mat{c} \sim p_c(\Mat{c})$.
Here is $\nabla \log p$ is a pre-trained diffusion model $\Mat{\epsilon_{\phi}}(\Mat{x}, t, \Mat{y})$ and $\Mat{x}_t$ is a noisy version of the rendering given by camera pose $\Mat{c}$. 
$\Mat{x}_t = \alpha_t g(\Mat{\theta}, \Mat{c}) + \sigma_t \Mat{\epsilon}$. 
Updating $\Mat{\theta}$ as in Eq.~\eqref{eqn:sds} has been shown to minimize the evidence lower bound (ELBO) for the rendered images, see~\citet{wang2023score, xu2022neurallift}.

\paragraph{Variational Score Distillation (VSD).}
VSD~\citep{wang2023prolificdreamer} is introduced in ProlificDreamer, VSD improves upon SDS by deriving the following Wasserstein gradient flow \citep{villani2009optimal}: 
\begin{align} \label{eqn:vsd}
\nabla_{\Mat{\theta}}J_{VSD}(\Mat{\theta}) = &-\mean \left[ \omega(t) \frac{\partial g(\Mat{\theta}, \Mat{c})}{\partial \Mat{\theta}} \left(\sigma_t \nabla\log p_t(\Mat{x}_t | \Mat{y}) \right. \right. \nonumber \\
& \hspace{4em} \left. \left. - \sigma_t \nabla\log q_t(\Mat{x}_t | \Mat{c}) \vphantom{)} \right) \vphantom{\frac{\partial g(\Mat{\theta}, \Mat{c})}{\partial \Mat{\theta}}} \right].
\end{align}
Similarly, $\Mat{x}_t = \alpha_t g(\Mat{\theta}, \Mat{c}) + \sigma_t \Mat{\epsilon}$ is the noisy observation of the rendered image.
In contrast to SDS, VSD introduces a new score function of the noisy rendered images conditioned on the camera pose $\mathbf{c}$. To obtain this score, \citet{wang2023prolificdreamer} fine-tunes a diffusion model using images rendered from the 3D scene as follows:
\begin{align} \label{eqn:lora}
\min_{\Mat{\psi}} \mean \left[ \omega(t) \lVert \Mat{\epsilon_{\psi}}(\alpha_t g(\Mat{\theta}, \Mat{c}) + \sigma_t \Mat{\epsilon}, t, \Mat{c}, \Mat{y}) - \Mat{\epsilon} \rVert_2^2 \right],
\end{align}
where $\Mat{\epsilon_{\psi}}(\Mat{x}, t, \Mat{c}, \Mat{y})$ is the noise estimator
of $\nabla\log q_t(\Mat{x}_t | \Mat{c})$ as in diffusion models. 
As proposed in ProlificDreamer, 
$\Mat{\psi}$ is parameterized by LoRA~\citep{hu2021lora} and initialized from a pre-trained diffusion model same as $\nabla\log p_t$.

\section{Revealing Mode Collapse in Score Distillation}
\label{sec:mode_collapse}

Despite the remarkable performance of SDS and VSD in 3D asset generation, it is widely observed that the synthesized objects suffer from ``\textit{Janus}" artifacts.
Janus artifacts refer to the generated 3D scene containing multiple canonical views (the most representative perspective of the object such as the frontal face).
In earlier works, \citet{hong2023debiasing} and \citet{huang2023dreamtime} attribute this problem to unimodality of the learned 2D image distribution since the training data for the diffusion models are naturally biased to the most commonly seen views per each category.
In this section, we examine extant distillation schemes from a statistical view, which has been overlooked in previous literature.

In principle, natural 2D images can be seen as random projections of 3D scenes.
Score distillation matches the image distribution generated by randomly sampled views with a text-conditioned image distribution to recover the underlying 3D representation.
Hence, Janus artifact, in which each view becomes uniform and identical to the most commonly seen views, can be interpreted as a manifestation of distribution collapse to samples within the high-density region.
Such distribution degeneration essentially corresponds to the statistical phenomenon \textit{mode collapse}, which happens when an optimized distribution fails to characterize the data diversity and concentrates on a single type of output \citep{goodfellow2014generative, salimans2016improved, metz2016unrolled, arjovsky2017wasserstein, srivastava2017veegan}.

Below we theoretically reveal why SDS and VSD are prone to mode collapse.
As shown in \citet{poole2022dreamfusion, wang2023prolificdreamer}, SDS and VSD equals to the gradient of the following Kullback-Leibler (KL) divergence, i.e., $J_{SDS}(\Mat{\theta}) = J_{VSD}(\Mat{\theta}) = J_{KL}(\Mat{\theta})$ up to an additive constant:
\begin{align} \label{eqn:kl_orig}
J_{KL}(\Mat{\theta}) = \mean \left[\Omega(t) \KL(q_t^{\Mat{\theta}}(\Mat{x}_t | \Mat{c}, \Mat{y}) \Vert p_t(\Mat{x}_t | \Mat{y})) \right],
\end{align}
where $\Omega(t) = \omega(t)\sigma_t / \alpha_t$ and the expectation is taken over $t \sim \uniform[0, T]$ and $\Mat{c} \sim p_c(\Mat{c})$.
Here $p_t(\Mat{x}_t | \Mat{y}) = \int p_0(\Mat{x}_0 | \Mat{y})  \gauss(\Mat{x}_t | \alpha_t\Mat{x}_0, \sigma_t^2 \Mat{I}) d\Mat{x}_0$ is the image distribution perturbed by Gaussian noises, while $q_t^{\Mat{\theta}}(\Mat{x}_t | \Mat{c}, \Mat{y}) = \int q^{\Mat{\theta}}_0(\Mat{x}_0 | \Mat{c})  \gauss(\Mat{x}_t | \alpha_t\Mat{x}_0, \sigma_t^2 \Mat{I}) d\Mat{x}_0$ models the image distribution generated by 3D parameter $\Mat{\theta}$ with respect to camera pose $\Mat{c}$ and diffused by Gaussian distribution.
As shown by \citet{wang2023prolificdreamer}, $J_{KL}(\Mat{\theta}) = 0$ implies $q^{\Mat{\theta}}_0(\Mat{x}_0 | \Mat{c}) = p(\Mat{x}_0 | \Mat{y})$, i.e., the distribution of synthesized views satisfy the text-conditioned image distribution.

However, it has not escaped from our notice that $q^{\Mat{\theta}}_0(\Mat{x}_0 | \Mat{c}) = \delta(\Mat{x}_0 - g(\Mat{\theta}, \Mat{c}))$ is a Dirac distribution for both SDS and VSD.
This causes the original KL divergence minimization (Eq. \ref{eqn:kl_orig}) degenerate to a Maximal Likelihood Estimation (MLE) problem:
\begin{align} \label{eqn:kl_equal_mle}
\begin{split}
J_{KL}(\Mat{\theta}) = \underbrace{-\mean \left[\Omega(t) \mean_{\Mat{x_t} \sim q_t^{\Mat{\theta}}(\Mat{x}_t | \Mat{c}, \Mat{y})} \log p_t(\Mat{x}_t | \Mat{y}) \right]}_{J_{MLE}(\Mat{\theta})} \\
- \underbrace{\mean \left[\Omega(t) H[q_t^{\Mat{\theta}}(\Mat{x}_t | \Mat{c}, \Mat{y})]\right]}_{const.},
\end{split}
\end{align}
where $H[q_t^{\Mat{\theta}}(\Mat{x}_t | \Mat{y})] = -\mean_{\Mat{x_t} \sim q_t^{\Mat{\theta}}(\Mat{x}_t | \Mat{c}, \Mat{y})} [\log q_t^{\Mat{\theta}}(\Mat{x}_t | \Mat{c}, \Mat{y})]$ denotes the entropy of $q_t^{\Mat{\theta}}(\Mat{x}_t | \Mat{y})$, which turns out to be a constant because $q_t^{\Mat{\theta}}(\Mat{x}_t | \Mat{c}, \Mat{y}) = \gauss(\Mat{x}_t | \alpha_t g(\Mat{\theta}, \Mat{c}), \sigma_t^2 \Mat{I})$ which has fixed entropy once $t$, $\Mat{\theta}$ and $\Mat{c}$ have been specified.
See full derivation in Appendix \ref{sec:proofs_kl}.

Note that Eq. \ref{eqn:kl_equal_mle} signifies $J_{KL}(\Mat{\theta}) = J_{MLE}(\Mat{\theta})$ up to an additive constant, hence $J_{KL}(\Mat{\theta})$ shares all minima with $J_{MLE}(\Mat{\theta})$.
It is known that likelihood maximization is more prone to mode collapse.
Intuitively, minimizing $J_{MLE}(\Mat{\theta})$ seeks each view \textit{independently} to have the maximum log-likelihood on the image distribution $p(\Mat{x}_0 | \Mat{y})$.
Since $p(\Mat{x}_0 | \Mat{y})$ is usually unimodal and peaks at the canonical view, each view of the scene will collapse to the same local minimum, resulting in Janus artifact (see Fig. \ref{fig:mle_vs_kl}).
We postulate that the existing distillation strategies may be inherently limited by their log-likelihood seeking behaviors, which are more susceptible to mode collapse, especially with biased image distributions.

\begin{figure}[t]
    \centering
    \includegraphics[width=\linewidth]{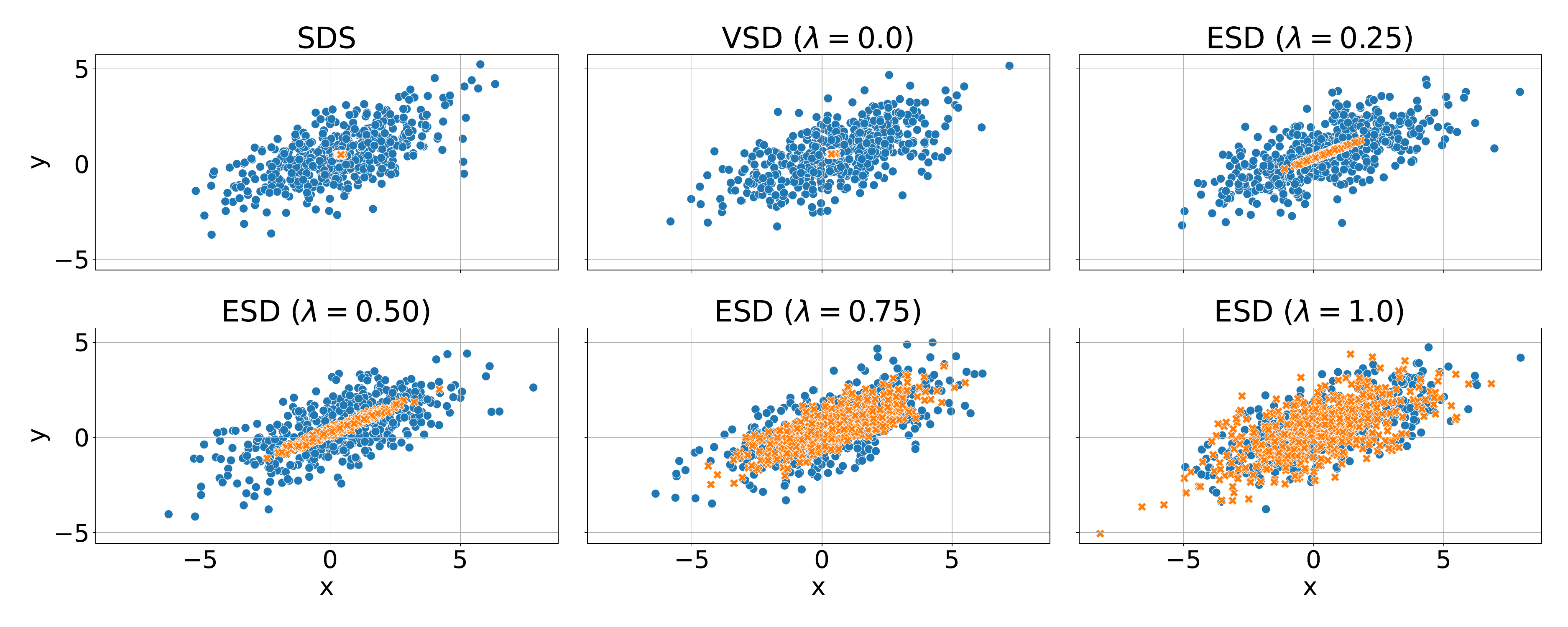}
    \vspace{-2.5em}
    \caption{\small \textbf{Gaussian Example.} To illustrate the effects of entropy regularization, we leverage SDS, VSD and ESD to fit a 2D Gaussian distribution. The \textcolor{blue}{blue} points are sampled from the ground-truth distribution while the \textcolor{orange}{orange} points are from the fitted distribution.}
    \label{fig:gaussian_example}
    \vspace{-1em}
\end{figure}

\section{Entropy Regularized Score Distillation} \label{sec:method}

\begin{algorithm*}[t]
\caption{ESD: Entropic score distillation for text-to-3D generation}
\label{alg:main}
\begin{algorithmic}
\STATE {\bf Input}: A diffusion model $\Mat{\epsilon_{\phi}}(\Mat{x}, t, \Mat{y})$; learnable 3D parameter $\Mat{\theta}$; coefficient $\lambda$; text prompt $\Mat{y}$.
\STATE Initialize $\Mat{\psi}$ for another diffusion model $\Mat{\epsilon_{\psi}}(\Mat{x}, t, \Mat{y})$  with the parameter $\Mat{\phi}$ specified in diffusion model $\Mat{\epsilon_{\phi}}(\Mat{x}, t, \Mat{y})$, parameterized with LoRA. 
\WHILE{not converged}
\STATE Randomly sample a camera pose $\Mat{c} \sim p_c$ and render a view $\Mat{x}_0 = g(\Mat{\theta}, \Mat{c})$ from $\Mat{\theta}$.
\STATE Sample a $t \sim \uniform[0, T]$ and add Gaussian noise $\Mat{\epsilon} \sim \gauss(\Mat{0}, \Mat{I})$: $\Mat{x}_t = \alpha_t \Mat{x}_0 + \sigma_t \Mat{\epsilon}$.
\STATE $\Mat{\theta} \leftarrow \Mat{\theta} + \eta_1 \left[ \omega(t) \frac{\partial g(\Mat{\theta}, \Mat{c})}{\partial \Mat{\theta}} (\Mat{\epsilon_{\phi}}(\Mat{x}_t, t, \Mat{y}) - \lambda \Mat{\epsilon_{\psi}}(\Mat{x}_t, t, \Mat{\emptyset}, \Mat{y}) - (1 - \lambda) \Mat{\epsilon_{\psi}}(\Mat{x}_t, t, \Mat{c}, \Mat{y}) \right] \hfill \refstepcounter{equation}(\theequation)\label{eqn:esd_cfg_alg}$
\STATE With probability $1-p_{\emptyset}$, $\Mat{\psi} \leftarrow \Mat{\psi} - \eta_2 \nabla_{\Mat{\psi}} \left[ \omega(t) \lVert \Mat{\epsilon_{\psi}}(\Mat{x}_t, t, \Mat{c}, \Mat{y}) - \Mat{\epsilon} \rVert_2^2 \right]$.
\STATE Otherwise, $\Mat{\psi} \leftarrow \Mat{\psi} - \eta_2 \nabla_{\Mat{\psi}} \left[ \omega(t) \lVert \Mat{\epsilon_{\psi}}(\Mat{x}_t, t, \Mat{\emptyset}, \Mat{y}) - \Mat{\epsilon} \rVert_2^2 \right]$.
\ENDWHILE
\STATE {\bf Return} $\Mat{\theta}$
\end{algorithmic}
\end{algorithm*}

\subsection{Entropic Score Distillation}
\label{sec:esd}

In this section, we highlight the importance of the entropy in score distillation.
It is known that higher entropy implies the corresponding distribution could cover a larger support of the ambient space and thus increase the sample diversity.
In Eq. \ref{eqn:kl_equal_mle}, the entropy term is shown to diminish in the training objective, which causes each generated view to lack diversity and collapse to a single image with the highest likelihood.

To this end, we propose to bring in an entropy regularization to $J_{MLE}(\Mat{\theta})$ for boosting the view diversity.
Since $q_t^{\Mat{\theta}}(\Mat{x}_t | \Mat{c}, \Mat{y})$ has constant entropy, we regularize entropy for the distribution $q_t^{\Mat{\theta}}(\Mat{x}_t | \Mat{y}) = \int q_t^{\Mat{\theta}}(\Mat{x}_t | \Mat{c}, \Mat{y}) p_c(\Mat{c}) d\Mat{c}$, which can be simulated by randomly sampling views from the 3D parameter $\Mat{\theta}$.
Consider the following objective:
\begin{align} \label{eqn:kl_ent}
\begin{split}
J_{Ent}(\Mat{\theta}, \lambda) = -\mean \left[\Omega(t) \mean_{\Mat{x_t} \sim q_t^{\Mat{\theta}}(\Mat{x}_t | \Mat{c}, \Mat{y})} \log p_t(\Mat{x}_t | \Mat{y}) \right] \\
- \lambda \mean \left[\Omega(t) H[q_t^{\Mat{\theta}}(\Mat{x}_t | \Mat{y})]\right],
\end{split}
\end{align}
where $\lambda$ is a hyper-parameter controlling the regularization strength. 
We note that without $H[q_t^{\Mat{\theta}}(\Mat{x}_t | \Mat{y})]$, each view is optimized independently and implicitly regularized by the underlying parameterization.
However, upon imposing $H[q_t^{\Mat{\theta}}(\Mat{x}_t | \Mat{y})]$, all views become explicitly correlated with each other, as they collectively contribute to the entropy computation.
Intuitively, $J_{Ent}(\Mat{\theta}, \lambda) = J_{MLE}(\Mat{\theta}) - \lambda \mean[\Omega(t) H[q_t^{\Mat{\theta}}(\Mat{x}_t | \Mat{y})]]$ seeks the maximal log-likelihood for each view while simultaneously enlarging the entropy for distribution $q_t^{\Mat{\theta}}(\Mat{x}_t | \Mat{y})$, which spans the support and encourages diversity across the rendered views.
To gain more insights, we present the following theoretical results:
\begin{theorem} \label{thm:ent_equal_kl}
For any $\lambda \in \real$ and $\Mat{\theta} \in \real^{D}$, we have $J_{Ent}(\Mat{\theta}, \lambda) = \lambda \mean_{t}[ \Omega(t) \KL(q_t^{\Mat{\theta}}(\Mat{x}_t | \Mat{y}) \Vert p_t(\Mat{x}_t | \Mat{y}))] + (1 - \lambda) \mean_{t, \Mat{c}}[\Omega(t) \KL(q_t^{\Mat{\theta}}(\Mat{x}_t | \Mat{c}, \Mat{y}) \Vert p_t(\Mat{x}_t | \Mat{y}))] + const.$
\end{theorem}
We prove Theorem \ref{thm:ent_equal_kl} in Appendix \ref{sec:proofs_cfg}.
Theorem \ref{thm:ent_equal_kl} implies that $J_{Ent}(\Mat{\theta}, \lambda)$ essentially equal to a combination of two types of KL divergences, where the former one minimizes the distribution discrepancy between $q_t^{\Mat{\theta}}(\Mat{x}_t | \Mat{y})$ and $p_t^{\Mat{\theta}}(\Mat{x}_t | \Mat{y})$ which marginalizes the camera pose within $q_t^{\Mat{\theta}}$, while the latter is the original KL divergence $J_{KL}(\Mat{\theta})$ adopted by SDS and VSD which takes expectation over $\Mat{c}$ out of KL divergence.

Next, we derive the gradient of $J_{Ent}(\Mat{\theta}, \lambda)$ that will be backpropagated to update the 3D representation. It can be obtained by path derivative and reparameterization trick:
\begin{align}
\nabla_{\Mat{\theta}} J_{Ent}(\Mat{\theta}, \lambda) = -\mean \left[ \omega(t) \frac{\partial g(\Mat{\theta}, \Mat{c})}{\partial \Mat{\theta}} \left( \sigma_t \nabla\log p_t(\Mat{x}_t | \Mat{y}) \right.\right. \label{eqn:esd} \\
\hspace{3em} - \left. \vphantom{\frac{\partial g(\Mat{\theta}, \Mat{c})}{\partial \Mat{\theta}}} \left. \lambda \sigma_t \nabla \log q_t^{\Mat{\theta}}(\Mat{x}_t | \Mat{y}) \right) \right]. \nonumber
\end{align}
The full derivation is deferred to Appendix \ref{sec:proofs_esd}.
We name this update rule as \textit{Entropic Score Distillation (ESD)}.
Note that ESD differs from VSD as its second score function does not depend on the camera pose.

\subsection{Classifier-Free Guidance Trick}
\label{sec:cfg_trick}

Similar to SDS and VSD, we approximate $\nabla\log p_t(\Mat{x}_t | \Mat{y})$ via a pre-trained diffusion model $\Mat{\epsilon}_{\Mat{\phi}}(\Mat{x}_t, t, \Mat{y})$.
However, $\nabla \log q_t^{\Mat{\theta}}(\Mat{x} | \Mat{y})$ is not readily available.
We found that directly fine-tuning a pre-trained diffusion model using rendered images to approximate $\nabla \log q_t^{\Mat{\theta}}(\Mat{x} | \Mat{y})$, akin to ProlificDreamer,  does not yield robust performance.
We postulate this difficulty arises from the removal of the camera condition, increasing the complexity of the distribution to be fitted.

To tackle this problem, we recall the result in Theorem \ref{thm:ent_equal_kl} that $J_{Ent}(\Mat{\theta}, \lambda)$ can be written in terms of two KL divergence losses.
Therefore, its gradient can be decomposed as a weighted combination of their gradients, which correspond to unconditional and conditional score functions in terms of the camera pose $\Mat{c}$, respectively:
\begin{align} \label{eqn:esd_cfg}
\nabla_{\Mat{\theta}} J_{Ent}(\Mat{\theta}, \lambda) = -\mean \left[ \omega(t) \frac{\partial g(\Mat{\theta}, \Mat{c})}{\partial \Mat{\theta}} ( \sigma_t \nabla\log p_t(\Mat{x}_t | \Mat{y})  \right.  \\
- \left. \vphantom{\frac{\partial g(\Mat{\theta}, \Mat{c})}{\partial \Mat{\theta}}}\lambda \sigma_t \nabla \log q_t^{\Mat{\theta}}(\Mat{x}_t | \Mat{y})
- (1 - \lambda) \sigma_t \nabla \log q_t^{\Mat{\theta}}(\Mat{x}_t | \Mat{c}, \Mat{y}) ) \right]. \nonumber
\end{align}
We formally prove Eq. \ref{eqn:esd_cfg} in Appendix \ref{sec:proofs_cfg}.
With the above formulation, ESD can be implemented via the Classifier-Free Guidance (CFG) trick, which was initially proposed to balance the variety and quality of text-conditionally generated images from diffusion models \citep{ho2022classifier}.
Algorithm \ref{alg:main} outlines the computation paradigm of ESD, in which we surrogate score functions in Eq. \ref{eqn:esd_cfg} with pre-trained and fine-tuned diffusion models (see Eq. \ref{eqn:esd_cfg_alg}), and takes random turns with a probability $p_{\emptyset}$ to balance the training of conditional and unconditional score functions, as suggested by \citet{ho2022classifier}.

\begin{figure*}[t]
    \centering
    \includegraphics[width=0.9\linewidth]{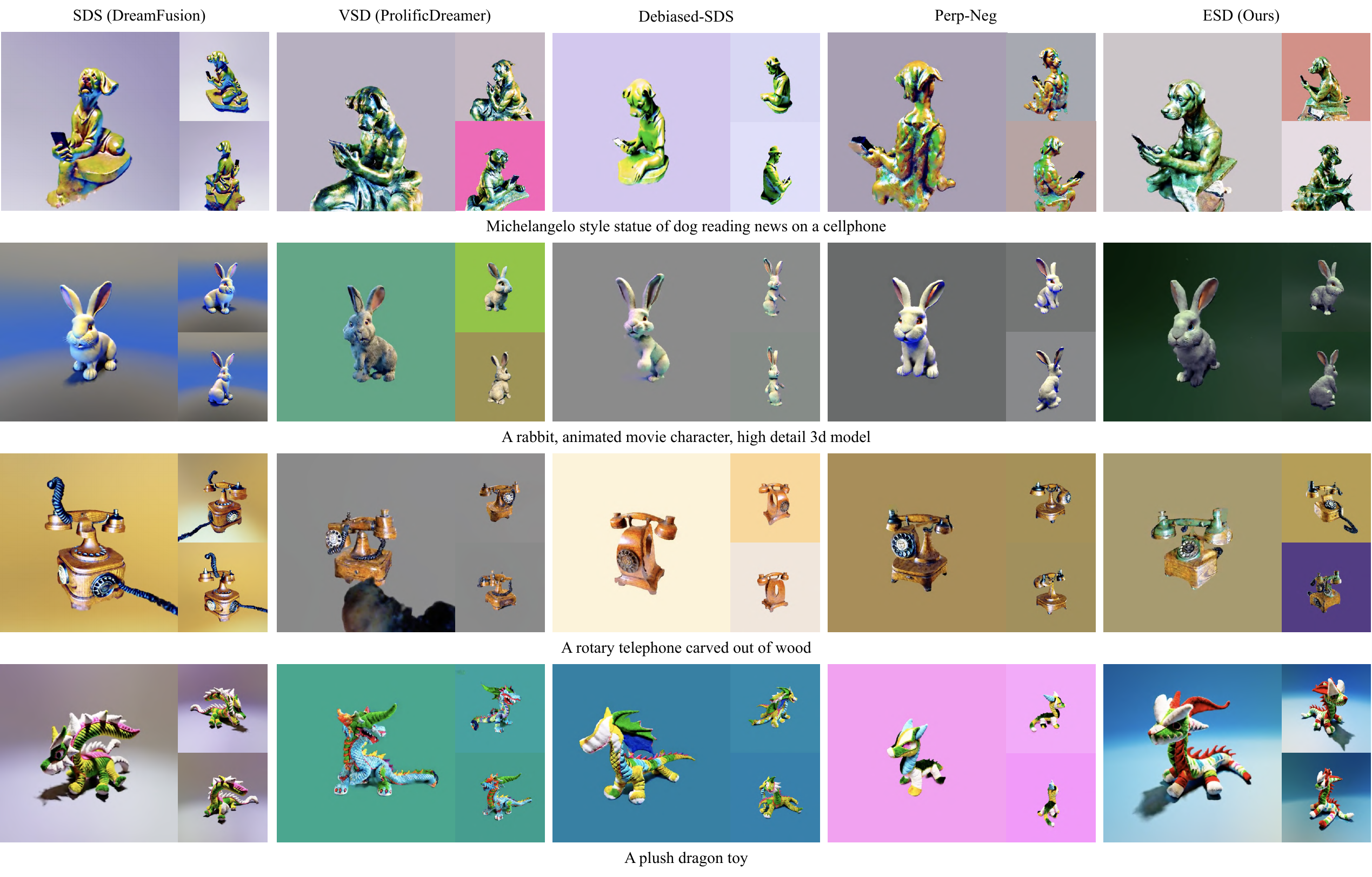}
    \caption{\small \textbf{Qualitative Results.} Our proposed outperforms all baselines in terms of better geometry and well-constructed texture details. Our results deliver photo-realistic and diverse rendered views, while baseline methods more or less suffer from the Janus problem. Best view in an electronic copy.}
    \label{fig:res_janus}
    \vspace{-1em}
\end{figure*}

\subsection{Discussion}

In VSD, the camera-conditioned score is believed to play a significant role in facilitating visual quality.
Intuitively, such conditioning can equip the tuned diffusion model with multi-view priors \citep{liu2023zero}.
Also, \citet{hertz2023delta} suggests such a method can be useful to stabilize the update of the implicit parameters.
However, ESD counters this argument by suggesting that the camera condition might not always be advantageous, particularly when the particle size is reduced to one.
In such cases, the resulting KL divergence provably degenerates to a likelihood maximization algorithm vulnerable to mode collapse.

It is noteworthy that, even though their subtle differences in implementation, the optimization objectives of ESD and VSD are fundamentally different (see Sec. \ref{sec:esd}).
ESD sets itself apart from VSD by incorporating entropy regularization, a crucial feature absent in VSD, aiming to augment diversity across views.
Despite originating from distinct objectives, our theoretical establishment allows for a straightforward implementation of ESD based on VSD using the CFG trick.

We provide an illustrative example by leveraging SDS, VSD and ESD (with different $\lambda$'s) to fit a 2D Gaussian distribution in Fig. \ref{fig:gaussian_example}.
With SDS and VSD, all samples are converged to the high-density area while ESD recovers the entire support of the distribution.
We provide more details and examples in Appendix \ref{sec:toy_examples}.

\begin{figure}[!t]
    \centering
    \includegraphics[width=0.95\linewidth]{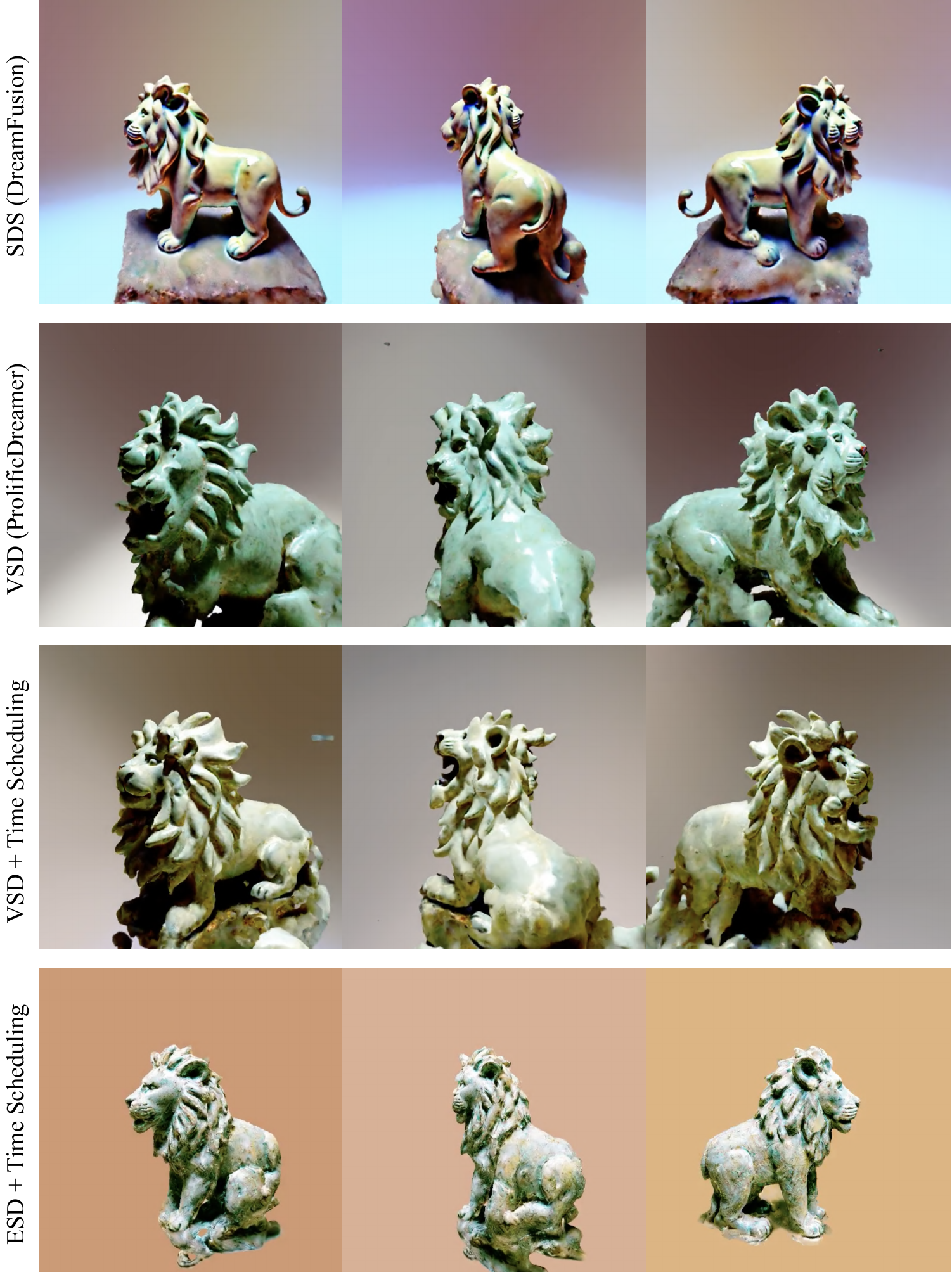}
    \caption{\small \textbf{Qualitative Results.} We combine our proposed ESD with timestep scheduling in DreamTime \citep{huang2023dreamtime} and compare it against baseline methods. Prompt: A caramic lion.}
    \label{fig:res_with_tp}
    \vspace{-1em}
\end{figure}

\section{Other Related Work}

\paragraph{Text-to-Image Diffusion Model.}
Text-to-image diffusion models~\citep{ramesh2021zero, ramesh2022hierarchical} are cornerstone components of text-to-3D generation.
It involves text embedding conditioning into the iterative denoising process.
Equipped with large-scale image-text paired datasets, many works ~\citep{nichol2021glide, ramesh2022hierarchical, imagen} scale up to tackle text-to-image generation. Among them, latent diffusion models attracted great interest in the open-source community since they reduced the computation cost by diffusing in the low-resolution latent space instead of directly in the pixel space. In addition, text-to-image diffusion models have also found applications in various computer vision tasks, including text-to-3D~\citep{poole2022dreamfusion,singer2023text}, image-to-3D~\citep{xu2022neurallift}, text-to-svg~\citep{jain2023vectorfusion}, text-to-video~\citep{singer2022make,khachatryan2023text2video}, etc.

\vspace{-1em}
\paragraph{3D Generation with 2D Priors.}
Well-annotated 3D data requires immense effort to collect. Instead, a line of research studies on how to learn 3D generative models using 2D supervision.
Early attempts, including pi-GAN~\citep{ranftl2021vision}, EG3D~\citep{chan2022efficient}, GRAF~\citep{schwarz2020graf}, GIRAFFE~\citep{niemeyer2021giraffe}, adopt adversarial loss between the rendered images and natural images.
DreamField \cite{jain2022zero} leverages CLIP to align NeRF with text prompts.
More recently, with the rapid development of text-to-image diffusion models, diffusion-based image priors have attracted increasing interest, and score distillation has then become the dominant technique.
Pioneer works DreamFusion~\citep{poole2022dreamfusion} and ProlificDreamer~\citep{wang2023prolificdreamer} have been introduced in detail in Sec. \ref{sec:prelim}.
Their concurrent work SJC~\citep{wang2023score} derives the score Jacobian chaining method from another theoretical viewpoint of Perturb and Average Scoring.
Even though diffusion models directly trained with 3D data nowadays demonstrate largely improved results \cite{shi2023mvdream, liu2023syncdreamer}, score distillation still plays a pivotal role in ensuring view consistency.

\vspace{-1em}
\paragraph{Techniques to Improve Score Distillation.} 
Providing the empirical promise of score distillation, there have been numerous techniques proposed to improve its effectiveness.
Magic3D \cite{lin2023magic3d} and Fantasia3D \cite{chen2023fantasia3d} utilize mesh and DMTet \cite{shen2021deep} to disentangle the optimization of geometry and texture. TextMesh \cite{tsalicoglou2023textmesh} and 3DFuse \cite{seo2023let} use depth-conditioned text-to-image diffusion priors that support geometry-aware texturing.
Score debiasing\citep{hong2023debiasing} and Perp-Neg~\citep{armandpour2023re} study to refine the text prompts for a better 3D generation.
DreamTime~\citep{huang2023dreamtime} and RED-Diff~\citep{mardani2023variational} investigate the timestep scheduling in the score distillation process.
HIFA~\citep{zhu2023hifa} adopts multiple diffusion steps for distillation. 
Score distillation also works with auxiliary losses, including CLIP loss~\citep{xu2022neurallift} and adversarial loss~\citep{shao2023control4d,chen2023it3d}.

\section{Evaluation Metrics} \label{sec:metrics}

In this section, we introduce four information-theoretic metrics to numerically evaluate the generated 3D results with a particular focus on identifying Janus artifacts or mode collapse.
The metrics we propose comprehensively cover four aspects: 1) the relevance with the text prompts, 2) distribution fitness, 3) rendering quality, and 4) view diversity.

\paragraph{CLIP Distance.}
We compute the average distance between rendered images and the text embedding to reflect the relevance between generated results and the specified text prompt.
Specifically, we render $N$ views from the generated 3D representations, and for each view, we obtain an embedding vector through the image encoder of a CLIP model \citep{wang2022clip}.
In the meantime, we compute the text embedding utilizing the text encoder.
The CLIP distance is computed as the one minus cosine similarity between the image embeddings and text embeddings averaged over all views.

\paragraph{Fr\'echet inception distance (FID).}
As shown in Sec. \ref{sec:mode_collapse} and \ref{sec:method}, score distillation essentially matches distributions via KL divergence.
Hence, it becomes reasonable to employ FID to measure the distance between the image distribution $q^{\Mat{\theta}}(\Mat{x}_0 | \Mat{y})$ generated by randomly rendering 3D representation and the text-conditioned image distribution $p(\Mat{x}_0 | \Mat{y})$ modeled by a diffusion model.
We sample $N$ images using pre-trained latent diffusion model given text prompts as the ground truth image dataset, and render $N$ views uniformly distributed over a unit sphere from the optimized 3D scene as the generated image dataset.
Then standard FID \citep{heusel2017gans} is computed between these two sets of images.
Note that FID is known to be effective in quantitatively identifying mode collapse.

\paragraph{Inception Quality and Variety.}
Thanks to our established connection with mode collapse, we know that Janus problem is due to a lack of sample diversity.
Inspired by Inception Score (IS) \citep{salimans2016improved}, we utilize entropy-related metrics to reflect the generated image quality and diversity.
We propose Inception Quality (IQ) and Inception Variety (IV), formulated as below:
\begin{align}
IQ(\Mat{\theta}) &= \mean_{\Mat{c}} \left[ H[p_{cls}(\Mat{y} | g(\Mat{\theta}, \Mat{c}))] \right], \\
IV(\Mat{\theta}) &= H[\mean_{\Mat{c}} [p_{cls}(\Mat{y} | g(\Mat{\theta}, \Mat{c})]],
\end{align}
\looseness=-1
where $p_{cls}(\Mat{y} | \Mat{x})$ is a pre-trained classifier.
IQ computes the average entropy of the label logits predicted for all rendered views, while IV computes the entropy of the averaged label logits of all rendered views.
Intuitively, the smaller IQ means highly confident classification results on rendered views, which also indicates better visual quality of generated 3D assets. 
In the meanwhile, the higher IV signifies that each rendered view is likely to have a distinct label prediction, meaning the 3D creation has higher view diversity.
Note that IV upper bounds IQ due to Jensen inequality.
So we can define Inception Gain $IG = (IV - IQ) / IQ$, which characterizes the information gain brought by knowing where the camera pose is, namely the improvement of distinguishability among different views.
\begin{figure*}
    \centering
    \includegraphics[width=0.9\linewidth]{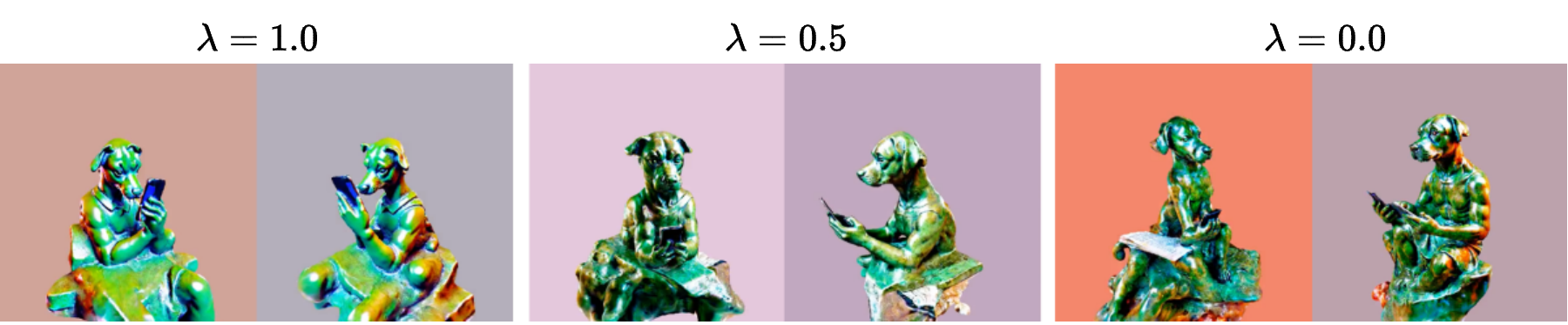}
    \vspace{-0.5em}
    \caption{\textbf{Ablation Studies on $\lambda$.} We investigate the choice of different entropy regularization strength $\lambda$. Prompt: Michelangelo-style statue of dog reading news on a cellphone.}
    \label{fig:ablation_cfg}
\end{figure*}

\section{Experiments}
\label{sec:expr}

\paragraph{Settings.}
In this section, we empirically validate the effectiveness of our proposal.
The chosen prompts are targeted at objects with clearly defined canonical views, posing a challenge for existing methods.
Our baseline approaches include SDS (DreamFusion) \cite{poole2022dreamfusion} and VSD (ProlificDreamer) \cite{wang2023prolificdreamer}, as well as two methods dedicated to solving Janus problem: Debiased-SDS \cite{hong2023debiasing} and Perp-Neg \cite{armandpour2023re}.
For fair comparison, all experiments are benchmarked under the open-source \href{https://github.com/threestudio-project/threestudio}{threestudio} framework.
Geometry refinement \cite{wang2023prolificdreamer} is adopted for all distillation schemes.
Please refer to Appendix \ref{sec:expr_details} for more implementation details.

\paragraph{Qualitative Comparison.}
We present qualitative comparisons in Fig.~\ref{fig:res_janus}. We encourage interested readers to Appendix \ref{sec:more_vis_res} for more results and our \href{{https://vita-group.github.io/3D-Mode-Collapse/}}{project page} for videos.
It is clearly shown that our proposed ESD delivers more precise geometry with the Janus problem alleviated.
In comparison, the results presented by SDS and VSD all contain more or less corrupted geometry with multi-face structures.
Debiased-SDS and Perp-Neg are shown to be effective for some text prompts, while not so consistent as ESD.
Additionally, we find that ESD can work particularly well when combined with the time-prioritized scheduling proposed in DreamTime~\citep{huang2023dreamtime}, as shown in Fig.~\ref{fig:res_with_tp}. 
This means ESD is orthogonal to many other methods and can cooperate with them to further reduce Janus artifacts.
\vspace{-0.5em}

\begin{table}[t]
\vspace{-1em}
\centering
\caption{\small \textbf{Quantitative Comparisons.} $(\downarrow)$ means the lower the better, and $(\uparrow)$ means the higher the better.}
\label{tab:metric}
\vspace{-0.5em}
\resizebox{\linewidth}{!}{
\begin{tabular}{c|cccccc}
\toprule
& CLIP ($\downarrow$) & FID ($\downarrow$) & IQ ($\downarrow$) & IV ($\uparrow$) & IG ($\uparrow$) & SR ($\uparrow$) \\
\hline
SDS & 0.737 & 291.860 & 4.295 & \textbf{4.8552} & 0.123 & 15.00\% \\
VSD & 0.725 & 265.141 & 3.149 & 3.5712 & 0.137 & 19.17\% \\
ESD & \textbf{0.714} & \textbf{235.915} & \textbf{3.135} & 4.0314 & \textbf{0.327} & \textbf{55.83\%} \\
\bottomrule
\end{tabular}
}
\vspace{-1em}
\end{table}
\paragraph{Quantitative Comparison.}
With metrics proposed in Sec. \ref{sec:metrics}, we numerically evaluate our method and baselines across 120 text prompts provided in \cite{wu2024gpt}.
We additionally involve \underline{S}uccessful generation \underline{R}ate (SR) based on human evaluation.
The results are presented in Tab.~\ref{tab:metric}.
We observe that among all metrics, ESD reaches the best CLIP score, FID, and IG.
More importantly, ESD achieves the optimal balance between view quality and diversity as shown by IQ and IV.
Whereas, SDS suffers from low image quality with high IQ and VSD is limited by insufficient view variety with low IV.
The superior IG of ESD indicates that views inside the generated scene are distinguishable rather than collapsing to be the same.
We defer the breakdown table for numerical evaluation on examples in Fig. \ref{fig:res_janus}, human evaluation criteria, and the standard deviation of metrics to Appendix \ref{sec:full_metrics}.

\paragraph{Ablation Studies}
We conduct ablation studies on the choice of $\lambda$ (\ie CFG weights) in Fig.~\ref{fig:ablation_cfg}.
We demonstrate that $\lambda$ can adjust ESD's preference toward view- quality or diversity.
When set to one, the produced Janus-free result albeits with fewer realistic details in the textures. Conversely, when set to zero, ESD equates to VSD, and the Janus problem emerges again. We empirically find that choosing $\lambda$ around 0.5 yields the best result, balancing fine-grained textures and well-constructed geometry.
We also implement ESD by directly fitting the score function $\nabla\log q_t^{\Mat{\theta}}(\Mat{x}_t | \Mat{y})$ without camera pose conditioning to validate the suggested implementation by CFG trick.
We show in Fig. \ref{fig:ablation_impl} that this optimization scheme is unstable. As training proceeds, the gradient explodes, and the optimized texture overflows.

\begin{figure}[!t]
    \centering
    \vspace{-1em}
    \includegraphics[width=\linewidth]{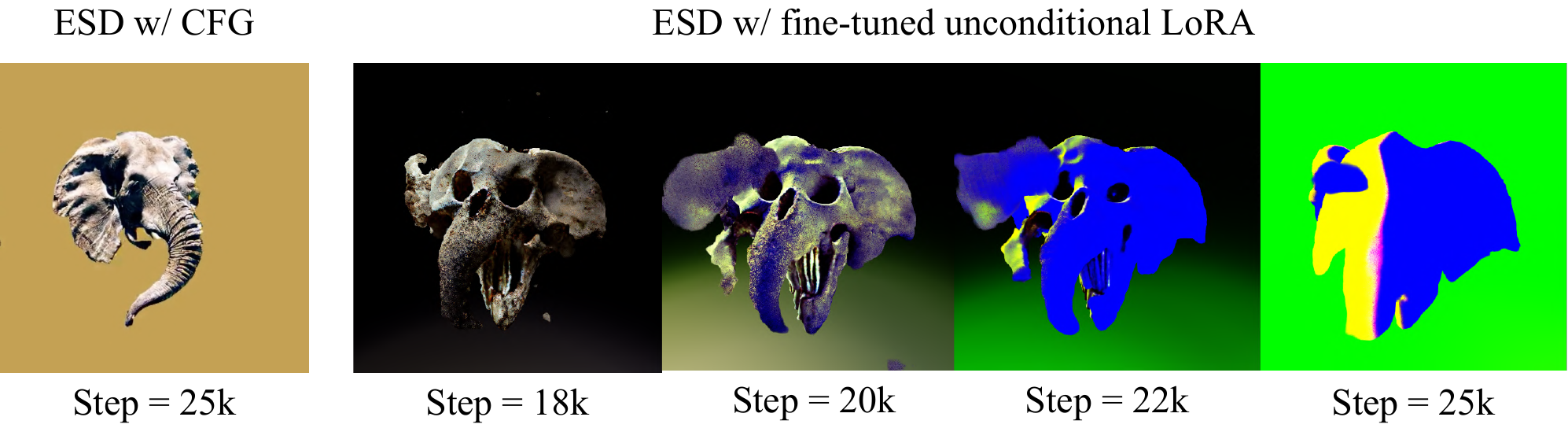}
    \vspace{-1.5em}
    \caption{\small \textbf{Ablation on Implementations.} The successfully generated result is obtained via our suggested CFG trick while the diverged result is yielded by fitting the unconditioned score function in Eq. \ref{eqn:esd} via LoRA. Prompt: an elephant skull.}
    \label{fig:ablation_impl}
    \vspace{-1em}
\end{figure}

\section{Conclusion}

In this paper, we reveal that existing score distillation methods degenerate to maximal likelihood seeking on each view independently, leading to the mode collapse problem. We identify that re-establishing the entropy term in the variational objective brings a new update rule, called Entropic Score Distillation (ESD), which is theoretically equivalent to adopting classifier-free guidance trick upon variational score distillation.
ESD maximizes the entropy of the rendered image distribution, encouraging diversity across views and mitigating the Janus problem.

\subsubsection*{Acknowledgments}
P Wang is sincerely grateful for constructive feedback regarding this manuscript from Zhaoyang Lv, Xiaoyu Xiang, Amit Kumar, Jinhui Xiong, and Varun Nagaraja.
P Wang also thanks Ruisi Cai for providing decent visual materials for illustration purposes.
Any statements, opinions, findings, and conclusions or recommendations expressed in this material are those of the authors and do not necessarily reflect the views of their employers or the supporting entities. 

{
    \small
    \bibliographystyle{ieeenat_fullname}
    \bibliography{main}
}

\clearpage
\setcounter{page}{1}
\appendix
\onecolumn

{\centering
\Large
\textbf{\thetitle}\\
\vspace{0.5em}Supplementary Material \\
\vspace{1.0em}
}
    
\section{Deferred Theory}
\label{sec:proofs}

We present deferred proofs and derivations in this section.
In the beginning, we justify several claimed properties of $J_{KL}$ (Eq. \ref{eqn:kl_orig}). Then we formally derive ESD (Eq. \ref{eqn:esd}) via our proposed objective $J_{Ent}$ (Eq. \ref{eqn:kl_ent}). Lastly, we prove that Classifier-Free Guidance trick (CFG) (Eq. \ref{eqn:esd_cfg}) can be used to implement ESD.

\subsection{Justification of Vanilla KL Divergence $J_{KL}$}
\label{sec:proofs_kl}

Let us consider KL divergence objective restated from Eq. \ref{eqn:kl_orig}:
\begin{align}
    J_{KL}(\Mat{\theta}) = \mean_{t \sim \uniform[0, T], \Mat{c} \sim p_c(\Mat{c})} \left[ \omega(t) \frac{\sigma_t}{\alpha_t} \KL(q^{\Mat{\theta}}_t(\Mat{x}_t | \Mat{c}, \Mat{y}) \Vert p_t(\Mat{x}_t | \Mat{y})) \right],
\end{align}
where we recall the notations: $\alpha_t, \sigma_t \in \real_{+}$ are time-dependent diffusion coefficients, $\Mat{c} \sim p_c(\Mat{c})$ is a camera pose drawn from a prior distribution over $\mathbb{SO}(3) \times \real^3$, and $g(\Mat{\theta}, \Mat{c})$ renders an image at viewpoint $\Mat{c}$ from the 3D representation $\Mat{\theta}$.
$p_t(\Mat{x}_t | \Mat{y})$ is the Gaussian diffused image distribution denoted as below:
\begin{align} \label{eqn:def_p_t}
p_t(\Mat{x}_t | \Mat{y}) = \int p_0(\Mat{x}_0 | \Mat{y}) \gauss(\Mat{x}_t | \alpha_t \Mat{x}_0, \sigma_t^2 \Mat{I}) d\Mat{x}_0,
\end{align}
where $p_0(\Mat{x}_0 | \Mat{y})$ is the text-conditioned distribution of clean images.
We also define $q_t(\Mat{x}_t | \Mat{c}, \Mat{y})$ as the Gaussian diffused distribution of rendered images:
\begin{align} \label{eqn:def_q_t}
q^{\Mat{\theta}}_t(\Mat{x}_t | \Mat{c}, \Mat{y}) = \int q^{\Mat{\theta}}_0(\Mat{x}_0 | \Mat{c}) \gauss(\Mat{x}_t | \alpha_t \Mat{x}_0, \sigma_t^2 \Mat{I}) d\Mat{x}_0,
\end{align}
where we assume $\Mat{x}_0$ is independent of text prompt $\Mat{y}$ given the camera pose and underlying 3D representation. Furthermore, we assume the rendering process has no randomness, thus $q^{\Mat{\theta}}_0(\Mat{x}_0 | \Mat{c}) = \delta(\Mat{x}_0 - g(\Mat{\theta}, \Mat{c}))$ can be written as a Dirac distribution.

Now, we can derive the gradient of $J_{KL}(\Mat{\theta})$, as summarized in the following lemma:
\begin{lemma}[Gradient of $J_{KL}$] \label{lem:grad_kl}
For any $\Mat{\theta}$, we have:
\begin{align}
\nabla_{\Mat{\theta}} J_{KL}(\Mat{\theta}) = -\mean_{t \sim \uniform[0, T], \Mat{c} \sim p_c(\Mat{c}), \Mat{\epsilon} \sim \gauss(\Mat{0}, \Mat{I})} \left[ \omega(t) \frac{\partial g(\Mat{\theta}, \Mat{c})}{\partial \Mat{\theta}} \sigma_t \nabla \log p_t(\Mat{x}_t | \Mat{y})\right],
\end{align}
where $\Mat{x}_t = \alpha_t \Mat{x}_0 + \sigma_t \Mat{\epsilon}$, and $\Mat{x}_0 = g(\Mat{\theta}, \Mat{c})$.
\end{lemma}
\begin{proof}
Due to the linearity of expectation, we have:
\begin{align}
&\nabla_{\Mat{\theta}} \mean_{t \sim \uniform[0, T], \Mat{c} \sim p_c(\Mat{c})} \left[ \omega(t) \frac{\sigma_t}{\alpha_t} \KL(q^{\Mat{\theta}}_t(\Mat{x}_t | \Mat{c}, \Mat{y}) \Vert p_t(\Mat{x}_t | \Mat{y})) \right] \\
&= \mean_{t \sim \uniform[0, T], \Mat{c} \sim p_c(\Mat{c})} \left[ \omega(t) \frac{\sigma_t}{\alpha_t} \nabla_{\theta} \KL(q^{\Mat{\theta}}_t(\Mat{x}_t | \Mat{c}, \Mat{y}) \Vert p_t(\Mat{x}_t | \Mat{y})) \right] \\
&= \mean_{t \sim \uniform[0, T], \Mat{c} \sim p_c(\Mat{c})} \left[ \omega(t) \frac{\sigma_t}{\alpha_t} \nabla_{\Mat{\theta}} \mean_{\Mat{x}_t \sim q^{\Mat{\theta}}_t(\Mat{x}_t | \Mat{c}, \Mat{y})} \left[ \log\frac{q^{\Mat{\theta}}_t(\Mat{x}_t | \Mat{c}, \Mat{y})}{p_t(\Mat{x}_t | \Mat{y})} \right]\right] \label{eqn:expand_grad_kl} 
\end{align}
Fixing $t$ and $\Mat{c}$, we apply reparameterization trick:
\begin{align}
& \nabla_{\Mat{\theta}} \mean_{\Mat{x}_t \sim q^{\Mat{\theta}}_t(\Mat{x}_t | \Mat{c}, \Mat{y})} \left[ \log\frac{q^{\Mat{\theta}}_t(\Mat{x}_t | \Mat{c}, \Mat{y})}{p_t(\Mat{x}_t | \Mat{y})} \right]
= \mean_{\Mat{\epsilon} \sim \gauss(\Mat{0}, \Mat{I})} \left[ \underbrace{\nabla_{\Mat{\theta}} \log q^{\Mat{\theta}}_t(\alpha_t g(\Mat{\theta}, \Mat{c}) + \sigma_t \Mat{\epsilon} | \Mat{c}, \Mat{y})}_{(a)} - \underbrace{\nabla_{\Mat{\theta}} \log p_t(\alpha_t g(\Mat{\theta}, \Mat{c}) +  \sigma_t \Mat{\epsilon} | \Mat{y})}_{(b)} \right].
\end{align}
Notice that $q^{\Mat{\theta}}_t(\alpha_t g(\Mat{\theta}, \Mat{c}) + \sigma_t \Mat{\epsilon} | \Mat{c}, \Mat{y}) = \gauss(\Mat{\epsilon} | \Mat{0}, \Mat{I})$ by substituting to Eq. \ref{eqn:def_q_t}, which is independent of $\Mat{\theta}$. Thus $(a) = \Mat{0}$.
For term (b), by chain rule, we have:
\begin{align}
\nabla_{\Mat{\theta}} \log p_t(\alpha_t g(\Mat{\theta}, \Mat{c}) + \sigma_t \Mat{\epsilon} | \Mat{y}) = \alpha_t \frac{\partial g(\Mat{\theta}, \Mat{c})}{\partial \Mat{\theta}} \nabla \log p_t(\alpha_t g(\Mat{\theta}, \Mat{c}) + \sigma_t \Mat{\epsilon} | \Mat{y}).  
\end{align}
Plugging back to Eq. \ref{eqn:expand_grad_kl}, we obtain:
\begin{align}
&\nabla_{\Mat{\theta}} J_{KL}(\Mat{\theta})
= -\mean_{t \sim \uniform[0, T], \Mat{c} \sim p_c(\Mat{c}), \Mat{\epsilon} \sim \gauss(\Mat{0}, \Mat{I})} \left[ \omega(t) \frac{\sigma_t}{\alpha_t} \cdot \alpha_t \frac{\partial g(\Mat{\theta}, \Mat{c})}{\partial \Mat{\theta}} \nabla \log p_t(\Mat{x}_t | \Mat{y}) \right] \\
&= -\mean_{t \sim \uniform[0, T], \Mat{c} \sim p_c(\Mat{c}), \Mat{\epsilon} \sim \gauss(\Mat{0}, \Mat{I})} \left[ \omega(t) \frac{\partial g(\Mat{\theta}, \Mat{c})}{\partial \Mat{\theta}} \sigma_t 
 \nabla \log p_t(\Mat{x}_t | \Mat{y}) \right],
\end{align}
where $\Mat{x}_t = \alpha_t g(\Mat{\theta}, \Mat{c}) + \sigma_t \Mat{\epsilon}$.
\end{proof}

Below we reproduce two results, which state both SDS  (Eq. \ref{eqn:sds}) and VSD  (Eq. \ref{eqn:vsd}) optimize for $J_{KL}$.
\begin{lemma}[SDS minimizes $J_{KL}$ \cite{poole2022dreamfusion}] \label{lem:grad_sds}
For any $\Mat{\theta}$, we have $J_{SDS}(\Mat{\theta}) = J_{KL}(\Mat{\theta}) + const.$
\end{lemma}
\begin{proof}
It is sufficient to show $\nabla_{\Mat{\theta}} J_{SDS}(\Mat{\theta}) = \nabla_{\Mat{\theta}} J_{KL}(\Mat{\theta})$. By expansion:
\begin{align}
&\nabla_{\Mat{\theta}}J_{SDS}(\Mat{\theta}) = -\mean_{t \sim \uniform[0, T], \Mat{c} \sim p_c(\Mat{c}), \Mat{\epsilon} \sim \gauss(\Mat{0}, \Mat{I})} \left[ \omega(t) \frac{\partial g(\Mat{\theta}, \Mat{c})}{\partial \Mat{\theta}} \left(\sigma_t \nabla \log p_t(\Mat{x}_t | \Mat{y}) - \Mat{\epsilon}\right) \right] \\
&= \underbrace{-\mean_{t \sim \uniform[0, T], \Mat{c} \sim p_c(\Mat{c}), \Mat{\epsilon} \sim \gauss(\Mat{0}, \Mat{I})} \left[ \omega(t) \frac{\partial g(\Mat{\theta}, \Mat{c})}{\partial \Mat{\theta}} \sigma_t \nabla \log p_t(\Mat{x}_t | \Mat{y}) \right]}_{\nabla_{\Mat{\theta}} J_{KL}(\Mat{\theta})} + \underbrace{\mean_{t \sim \uniform[0, T], \Mat{c} \sim p_c(\Mat{c}), \Mat{\epsilon} \sim \gauss(\Mat{0}, \Mat{I})} \left[ \omega(t) \sigma_t \frac{\partial g(\Mat{\theta}, \Mat{c})}{\partial \Mat{\theta}} \Mat{\epsilon} \right],}_{=\Mat{0}}
\end{align}
where the second term equals $\Mat{0}$ because $\Mat{\epsilon}$ is zero mean and sampled independently.
\end{proof}

\begin{lemma}[Single-particle VSD minimizes $J_{KL}$ \citep{wang2023prolificdreamer}] \label{lem:grad_vsd}
For any $\Mat{\theta}$, we have $J_{VSD}(\Mat{\theta}) = J_{KL}(\Mat{\theta}) + const.$
\end{lemma}
\begin{proof}
It is sufficient to show $\nabla_{\Mat{\theta}} J_{VSD}(\Mat{\theta}) = \nabla_{\Mat{\theta}} J_{KL}(\Mat{\theta})$. By a similar expansion:
\begin{align}
\nabla_{\Mat{\theta}}J_{VSD}(\Mat{\theta}) &= -\mean_{t \sim \uniform[0, T], \Mat{c} \sim p_c(\Mat{c}), \Mat{\epsilon} \sim \gauss(\Mat{0}, \Mat{I})} \left[ \omega(t) \frac{\partial g(\Mat{\theta}, \Mat{c})}{\partial \Mat{\theta}} \left(\sigma_t \nabla \log p_t(\Mat{x}_t | \Mat{y}) - \sigma_t \nabla \log q^{\Mat{\theta}}_t(\Mat{x}_t | \Mat{c}, \Mat{y}) \right) \right] \\
&= \underbrace{-\mean_{t \sim \uniform[0, T], \Mat{c} \sim p_c(\Mat{c}), \Mat{\epsilon} \sim \gauss(\Mat{0}, \Mat{I})} \left[ \omega(t) \frac{\partial g(\Mat{\theta}, \Mat{c})}{\partial \Mat{\theta}} \sigma_t \nabla \log p_t(\Mat{x}_t | \Mat{y}) \right]}_{\nabla_{\Mat{\theta}} J_{KL}(\Mat{\theta})} \\
&\quad\quad + \underbrace{\mean_{t \sim \uniform[0, T], \Mat{c} \sim p_c(\Mat{c}), \Mat{\epsilon} \sim \gauss(\Mat{0}, \Mat{I})} \left[ \omega(t) \frac{\partial g(\Mat{\theta}, \Mat{c})}{\partial \Mat{\theta}} \sigma_t \nabla \log q^{\Mat{\theta}}_t(\Mat{x}_t | \Mat{c}, \Mat{y}) \right]}_{=(a)}
\end{align}
Then we conclude the proof by showing $(a) = \Mat{0}$ due to the fact that the first-order moment of score functions equals zero:
\begin{align}
(a) &= \mean_{t \sim \uniform[0, T], \Mat{c} \sim p_c(\Mat{c})}  \left[ \omega(t)\frac{\sigma_t}{\alpha_t} \mean_{\Mat{x}_t \sim q^{\Mat{\theta}}_t(\Mat{x} | \Mat{c}, \Mat{y}) } \left[ \alpha_t \frac{\partial g(\Mat{\theta}, \Mat{c})}{\partial \Mat{\theta}} \nabla \log q^{\Mat{\theta}}_t(\Mat{x}_t | \Mat{c}, \Mat{y}) \right] \right] \\
&= \mean_{t \sim \uniform[0, T], \Mat{c} \sim p_c(\Mat{c})}  \left[ \omega(t)\frac{\sigma_t}{\alpha_t} \mean_{\Mat{x}_t \sim q^{\Mat{\theta}}_t(\Mat{x} | \Mat{c}, \Mat{y}) } \left[ \nabla_{\Mat{\theta}} \log q^{\Mat{\theta}}_t(\Mat{x}_t | \Mat{c}, \Mat{y}) \right] \right] \label{eqn:reverse_chain_rule} \\
&= \mean_{t \sim \uniform[0, T], \Mat{c} \sim p_c(\Mat{c})}  \left[ \omega(t)\frac{\sigma_t}{\alpha_t} \int \frac{\nabla_{\Mat{\theta}} q^{\Mat{\theta}}_t(\Mat{x}_t | \Mat{c}, \Mat{y})}{q^{\Mat{\theta}}_t(\Mat{x}_t | \Mat{c}, \Mat{y})} q^{\Mat{\theta}}_t(\Mat{x}_t | \Mat{c}, \Mat{y}) d\Mat{x}_t \right] \\
&= \mean_{t \sim \uniform[0, T], \Mat{c} \sim p_c(\Mat{c})}  \left[ \omega(t)\frac{\sigma_t}{\alpha_t} \nabla_{\Mat{\theta}} \int q^{\Mat{\theta}}_t(\Mat{x}_t | \Mat{c}, \Mat{y}) d\Mat{x}_t \right] = \Mat{0},
\end{align}
where we use change of variables by reversing the chain rule in Eq. \ref{eqn:reverse_chain_rule}, and the last step is because the integral equals one, which is independent of $\Mat{\theta}$.
\end{proof}
\begin{remark}
For multi-particle VSD, Lemma \ref{lem:grad_vsd} may not hold. This is because the reverse chain rule in Eq. \ref{eqn:reverse_chain_rule} is no longer applicable as $q^{\Mat{\theta}}_t(\Mat{x}_t | \Mat{c}, \Mat{y})$ also becomes a function of $\Mat{\theta}$.
\end{remark}

Finally, we show that optimizing $J_{KL}$ is equivalent to optimizing $J_{MLE}$ (Eq. \ref{eqn:kl_equal_mle}). First, recall that:
\begin{align}
J_{MLE}(\Mat{\theta}) = -\mean_{t \sim \uniform[0, T], \Mat{c} \sim p_c(\Mat{c})} \left[ \omega(t) \frac{\sigma_t}{\alpha_t} \mean_{\Mat{x}_t \sim q^{\Mat{\theta}}_t(\Mat{x}_t | \Mat{c}, \Mat{y})} \left[\log p_t(\Mat{x}_t | \Mat{y})\right] \right].
\end{align}
Then we state the following lemma:
\begin{lemma}[$J_{KL}$ is equivalent to maximal likelihood estimation] \label{lem:grad_mle}
For any $\Mat{\theta}$, we have $J_{MLE}(\Mat{\theta}) = J_{KL}(\Mat{\theta}) + const.$
\end{lemma}
\begin{proof}
Again, we show $\nabla_{\Mat{\theta}} J_{MLE}(\Mat{\theta}) = \nabla_{\Mat{\theta}} J_{KL}(\Mat{\theta})$:
\begin{align}
\nabla_{\Mat{\theta}}J_{MLE}(\Mat{\theta}) &= \nabla_{\Mat{\theta}} -\mean_{t \sim \uniform[0, T], \Mat{c} \sim p_c(\Mat{c})} \left[ \omega(t) \frac{\sigma_t}{\alpha_t} \mean_{\Mat{x}_t \sim q^{\Mat{\theta}}_t(\Mat{x}_t | \Mat{c}, \Mat{y})} \left[\log p_t(\Mat{x}_t | \Mat{y})\right] \right] \\
&= -\mean_{t \sim \uniform[0, T], \Mat{c} \sim p_c(\Mat{c})} \left[ \omega(t) \frac{\sigma_t}{\alpha_t} \mean_{\Mat{x}_t \sim q^{\Mat{\theta}}_t(\Mat{x}_t | \Mat{c}, \Mat{y})} \left[\nabla_{\Mat{\theta}} 
 \log p_t(\Mat{x}_t | \Mat{y})\right] \right] \\
&= -\mean_{t \sim \uniform[0, T], \Mat{c} \sim p_c(\Mat{c})} \left[ \omega(t) \frac{\sigma_t}{\alpha_t} \mean_{\Mat{x}_t \sim q^{\Mat{\theta}}_t(\Mat{x}_t | \Mat{c}, \Mat{y})} \left[\alpha_t \frac{\partial g(\Mat{\theta}, \Mat{c})}{\partial \Mat{\theta}} \nabla \log p_t(\Mat{x}_t | \Mat{y})\right] \right] \\
&= -\mean_{t \sim \uniform[0, T], \Mat{c} \sim p_c(\Mat{c}), \Mat{\epsilon} \sim \gauss(\Mat{0}, \Mat{I})} \left[ \omega(t) \frac{\partial g(\Mat{\theta}, \Mat{c})}{\partial \Mat{\theta}} \sigma_t \nabla \log p_t(\Mat{x}_t | \Mat{y}) \right],
\end{align}
where the last step is basic reparameterization of $\Mat{x}_t = \alpha_t \Mat{x}_0 + \sigma_t \Mat{\epsilon}$, and $\Mat{x}_0 = g(\Mat{\theta}, \Mat{c})$.
\end{proof}
As we argue in Sec. \ref{sec:mode_collapse} (Eq. \ref{eqn:kl_equal_mle}), the root reason $J_{KL}$ degenerates to $J_{MLE}$ is because the entropy term in $J_{KL}$ becomes a constant independent of $\Mat{\theta}$.

\subsection{Derivation of Entropic Score Distillation}
\label{sec:proofs_esd}

In this section, we derive the gradient for our entropy regularized objective (Eq. \ref{eqn:esd}). We restate the entropy regularized objective (Eq. \ref{eqn:kl_ent}) below:
\begin{align}
J_{Ent}(\Mat{\theta}, \lambda) = -\mean_{t \sim \uniform[0, T], \Mat{c} \sim p_c(\Mat{c})} \left[\omega(t) \frac{\sigma_t}{\alpha_t} \mean_{\Mat{x_t} \sim q_t^{\Mat{\theta}}(\Mat{x}_t | \Mat{c}, \Mat{y})} \log p_t(\Mat{x}_t | \Mat{y}) \right] - \lambda \mean_{t \sim \uniform[0, T]} \left[\omega(t)  \frac{\sigma_t}{\alpha_t} H[q_t^{\Mat{\theta}}(\Mat{x}_t | \Mat{y})]\right],
\end{align}
where the entropy term $H[q_t^{\Mat{\theta}}(\Mat{x}_t | \Mat{y})]$ is defined as:
\begin{align}
H\left[q^{\Mat{\theta}}_t(\Mat{x}_t | \Mat{y})\right] = -\mean_{\Mat{x}_t \sim q^{\Mat{\theta}}_t(\Mat{x}_t | \Mat{y})} \left[ \log q^{\Mat{\theta}}_t(\Mat{x}_t | \Mat{y}) \right],
\end{align}
and distribution $q^{\Mat{\theta}}_t(\Mat{x}_t | \Mat{y})$ is defined as:
\begin{align}
q^{\Mat{\theta}}_t(\Mat{x}_t | \Mat{y}) = \int q^{\Mat{\theta}}_t(\Mat{x}_t | \Mat{c}, \Mat{y}) p_c(\Mat{c}) d\Mat{c}.
\end{align}
Notice that $J_{Ent}(\Mat{\theta}, \lambda) = J_{MLE}(\Mat{\theta}) - \lambda \mean_{t \sim \uniform[0, T]} \left[\omega(t)  \frac{\sigma_t}{\alpha_t} H[q_t^{\Mat{\theta}}(\Mat{x}_t | \Mat{y})]\right]$, therefore, to derive Eq. \ref{eqn:esd}, we simply need the gradient of the entropy term:
\begin{lemma}[Gradient of entropy] \label{lem:grad_ent}
It holds that:
\begin{align}
\nabla_{\Mat{\theta}} H\left[q_t^{\Mat{\theta}}(\Mat{x}_t | \Mat{y})\right] = -\mean_{\Mat{c} \sim p_c(\Mat{c}), \Mat{\epsilon} \sim \gauss(\Mat{0}, \Mat{I})} \left[ \alpha_t \frac{\partial g(\Mat{\theta}, \Mat{c})}{\partial \Mat{\theta}} \nabla \log q^{\Mat{\theta}}_t(\Mat{x}_t | \Mat{y}) \right].
\end{align}
\end{lemma}
\begin{proof}
We expand entropy by reparameterization of $q_t^{\Mat{\theta}}(\Mat{x}_t | \Mat{y})$ as sampling two independent variables $\Mat{c}, \Mat{\epsilon}$:
\begin{align}
&\nabla_{\Mat{\theta}} H\left[q_t^{\Mat{\theta}}(\Mat{x}_t | \Mat{y})\right] = \nabla_{\Mat{\theta}} \mean_{\Mat{c} \sim p_c(\Mat{c}), \Mat{\epsilon} \sim \gauss(\Mat{0}, \Mat{I})} \left[-\log q_t^{\Mat{\theta}}(\alpha_t g(\Mat{\theta}, \Mat{c}) + \sigma_t \Mat{\epsilon} | \Mat{y})\right] \\
&= -\mean_{\Mat{c} \sim p_c(\Mat{c}), \Mat{\epsilon} \sim \gauss(\Mat{0}, \Mat{I})} \left.\left[\nabla_{\Mat{\theta}} \log q^{\Mat{\theta}}_t(\Mat{x}_t | \Mat{y}) + \alpha_t \frac{\partial g(\Mat{\theta}, \Mat{c})}{\partial \Mat{\theta}} \nabla_{\Mat{x}_t} \log q^{\Mat{\theta}}_t(\Mat{x}_t | \Mat{y}) \right]\right\vert_{\Mat{x}_t = \alpha_t g(\Mat{\theta}, \Mat{c}) + \sigma_t \Mat{\epsilon}} \label{eqn:path_derivative} \\
&= -\underbrace{\mean_{\Mat{x}_t \sim q^{\Mat{\theta}}_t(\Mat{x}_t | \Mat{y})} \left[\nabla_{\Mat{\theta}}\log q^{\Mat{\theta}}_t(\Mat{x}_t | \Mat{y})\right]}_{= (a)} - \mean_{\Mat{c} \sim p_c(\Mat{c}), \Mat{\epsilon} \sim \gauss(\Mat{0}, \Mat{I})} \left[\alpha_t \frac{\partial g(\Mat{\theta}, \Mat{c})}{\partial \Mat{\theta}} \nabla_{\Mat{x}_t} \log q^{\Mat{\theta}}_t(\alpha_t g(\Mat{\theta}, \Mat{c}) + \sigma_t \Mat{\epsilon} | \Mat{\theta}) \right],
\end{align}
where it is noteworthy that $\nabla_{\Mat{x}_t} \log q^{\Mat{\theta}}_t$ simply denotes the score function of $q^{\Mat{\theta}}_t$ by explicitly indicating the derivative is taken in terms of $\Mat{x}_t$. Eq. \ref{eqn:path_derivative} is obtained by path derivative.
It remains to show $(a) = \Mat{0}$. We recall that the first-order moment of a score function equals to zero:
\begin{align}
(a) &= \int \nabla_{\Mat{\theta}}\log q^{\Mat{\theta}}_t(\Mat{x}_t | \Mat{y}) q^{\Mat{\theta}}_t(\Mat{x}_t | \Mat{y}) d\Mat{x}_t = \int \frac{\nabla_{\Mat{\theta}} q^{\Mat{\theta}}_t(\Mat{x}_t | \Mat{y})}{q^{\Mat{\theta}}_t(\Mat{x}_t | \Mat{y})} q^{\Mat{\theta}}_t(\Mat{x}_t | \Mat{y}) d\Mat{x}_t \\
&= \nabla_{\Mat{\theta}} \int q^{\Mat{\theta}}_t(\Mat{x}_t | \Mat{y}) d\Mat{x}_t \\
&= \Mat{0},
\end{align}
where the last step involves a change of variable and the integral turns out to be independent of $\Mat{\theta}$.
\end{proof}
As a consequence, we can conclude the update rule yielded by Eq. \ref{eqn:kl_ent} in the following theorem:
\begin{theorem}[Entropic Score Distillation] \label{thm:grad_esd}
For any $\Mat{\theta}$ and $\lambda \in \real$, the following holds:
\begin{align}
\nabla_{\Mat{\theta}} J_{Ent}(\Mat{\theta}, \lambda) = -\mean_{t \sim \uniform[0, T], \Mat{c} \sim p_c(\Mat{c}), \Mat{\epsilon} \sim \gauss(\Mat{0}, \Mat{I})} \left[ \omega(t) \frac{\partial g(\Mat{\theta}, \Mat{c})}{\partial \Mat{\theta}} \left(\sigma_t \nabla \log p_t(\Mat{x}_t | \Mat{y}) - \lambda \sigma_t \nabla \log q^{\Mat{\theta}}_t(\Mat{x}_t | \Mat{y}) \right) \right],
\end{align}
where $\Mat{x}_t = \alpha_t \Mat{x}_0 + \sigma_t \Mat{\epsilon}$, and $\Mat{x}_0 = g(\Mat{\theta}, \Mat{c})$.
\end{theorem}
\begin{proof}
Since $J_{Ent}(\Mat{\theta}, \lambda) = J_{MLE}(\Mat{\theta}) - \lambda \mean_{t \sim \uniform[0, T]} \left[\omega(t)  \frac{\sigma_t}{\alpha_t} H[q_t^{\Mat{\theta}}(\Mat{x}_t | \Mat{y})]\right]$, by Lemma \ref{lem:grad_mle} and Lemma \ref{lem:grad_ent}:
\begin{align}
\nabla_{\Mat{\theta}} J_{Ent}(\Mat{\theta}, \lambda) &= \nabla_{\Mat{\theta}} J_{MLE}(\Mat{\theta}) - \lambda \mean_{t \sim \uniform[0, T]} \left[\omega(t)  \frac{\sigma_t}{\alpha_t} \nabla_{\Mat{\theta}} H[q_t^{\Mat{\theta}}(\Mat{x}_t | \Mat{y})]\right] \\
&= -\mean_{t \sim \uniform[0, T], \Mat{c} \sim p_c(\Mat{c}), \Mat{\epsilon} \sim \gauss(\Mat{0}, \Mat{I})} \left[ \omega(t) \frac{\partial g(\Mat{\theta}, \Mat{c})}{\partial \Mat{\theta}} \sigma_t \nabla \log p_t(\Mat{x}_t | \Mat{y}) \right] \\
&\quad \quad + \mean_{t \sim \uniform[0, T], \Mat{c} \sim p_c(\Mat{c}), \Mat{\epsilon} \sim \gauss(\Mat{0}, \Mat{I})} \left[ \omega(t) \frac{\partial g(\Mat{\theta}, \Mat{c})}{\partial \Mat{\theta}} \lambda \sigma_t \nabla \log q^{\Mat{\theta}}_t(\Mat{x}_t | \Mat{y}) \right],
\end{align}
by which we conclude the proof by merging two expectations.
\end{proof}

\subsection{Justification of Classifier-Free Guidance Trick}
\label{sec:proofs_cfg}

In this section, we first prove Theorem \ref{thm:ent_equal_kl} and show that CFG trick (Eq. \ref{eqn:esd_cfg}) can be utilized to implement ESD.
To begin with, we define another type of KL divergence as below:
\begin{align}
    \overline{J_{KL}} = \mean_{t \sim \uniform[0, T]} \left[\omega(t) \frac{\sigma_t}{\alpha_t} \KL(q^{\Mat{\theta}}_t(\Mat{x}_t | \Mat{y}) \Vert p_t(\Mat{x}_t | \Mat{y}))\right]
\end{align}
Then we present the following lemma, which represents the gradient of $\overline{J_{KL}}$:
\begin{lemma}[Gradient of $\overline{J_{KL}}$] \label{lem:grad_kl_ii}
It holds that:
\begin{align}
\nabla_{\Mat{\theta}} \overline{J_{KL}} = -\mean_{t \sim \uniform[0, T], \Mat{c} \sim p_c(\Mat{c}), \Mat{\epsilon} \sim \gauss(\Mat{0}, \Mat{I})} \left[ \omega(t) \frac{\partial g(\Mat{\theta}, \Mat{c})}{\partial \Mat{\theta}} \left(\sigma_t \nabla \log p_t(\Mat{x}_t | \Mat{y}) - \sigma_t \nabla \log q^{\Mat{\theta}}_t(\Mat{x}_t | \Mat{y}) \right) \right],
\end{align}
where $\Mat{x}_t = \alpha_t \Mat{x}_0 + \sigma_t \Mat{\epsilon}$, and $\Mat{x}_0 = g(\Mat{\theta}, \Mat{c})$.
\end{lemma}
\begin{proof}
We prove by showing $\overline{J_{KL}}$ is a special case of $J_{Ent}$ when setting $\lambda = 1$:
\begin{align}
& \overline{J_{KL}}(\Mat{\theta}) = \mean_{t \sim \uniform[0, T], \Mat{c} \sim p_c(\Mat{c}), \Mat{\epsilon} \sim \gauss(\Mat{0}, \Mat{I})} \left[\omega(t) \frac{\sigma_t}{\alpha_t} \log \frac{q^{\Mat{\theta}}_t(\Mat{x}_t | \Mat{y})}{p_t(\Mat{x}_t | \Mat{y}))}\right] \\
&= -\mean_{t \sim \uniform[0, T], \Mat{c} \sim p_c(\Mat{c}), \Mat{\epsilon} \sim \gauss(\Mat{0}, \Mat{I})} \left[\omega(t) \frac{\sigma_t}{\alpha_t} \log p_t(\Mat{x}_t | \Mat{y}))\right] + \mean_{t \sim \uniform[0, T], \Mat{c} \sim p_c(\Mat{c}), \Mat{\epsilon} \sim \gauss(\Mat{0}, \Mat{I})} \left[\omega(t) \frac{\sigma_t}{\alpha_t} \log q^{\Mat{\theta}}_t(\Mat{x}_t | \Mat{y})\right] \\
&= J_{MLE}(\Mat{\theta}) + \mean_{t \sim \uniform[0, T]} \mean_{\Mat{x}_t \sim q^{\Mat{\theta}}_t(\Mat{x}_t | \Mat{y})} \left[\omega(t) \frac{\sigma_t}{\alpha_t} \log q^{\Mat{\theta}}_t(\Mat{x}_t | \Mat{y})\right] \\
&= J_{MLE}(\Mat{\theta}) -  \mean_{t \sim \uniform[0, T]} \left[ \omega(t) \frac{\sigma_t}{\alpha_t} H\left[q^{\Mat{\theta}}_t(\Mat{x}_t | \Mat{y})\right]\right] = J_{Ent}(\Mat{\theta}, 1)
\end{align}
\end{proof}

Now we prove Theorem \ref{thm:ent_equal_kl} using previous results:
\begin{proof}[Proof of Theorem \ref{thm:ent_equal_kl}]
It is sufficient to show that $\nabla_{\Mat{\theta}} J_{Ent}(\Mat{\theta}, \lambda) = \lambda \nabla_{\Mat{\theta}} \overline{J_{KL}}(\Mat{\theta})  + (1-\lambda) \nabla_{\Mat{\theta}} J_{KL}(\Mat{\theta})$.
By Lemma \ref{lem:grad_kl}, \ref{lem:grad_kl_ii}, as well as Theorem \ref{thm:grad_esd}, we can obtain:
\begin{align}
\lambda \nabla_{\Mat{\theta}} \overline{J_{KL}}(\Mat{\theta})  + (1-\lambda) \nabla_{\Mat{\theta}} J_{KL}(\Mat{\theta}) &= -\lambda \mean_{t \sim \uniform[0, T], \Mat{c} \sim p_c(\Mat{c}), \Mat{\epsilon} \sim \gauss(\Mat{0}, \Mat{I})} \left[ \omega(t) \frac{\partial g(\Mat{\theta}, \Mat{c})}{\partial \Mat{\theta}} \left(\sigma_t \nabla \log p_t(\Mat{x}_t | \Mat{y}) - \sigma_t \nabla \log q^{\Mat{\theta}}_t(\Mat{x}_t | \Mat{y}) \right) \right] \\
&\quad -(1 - \lambda)\mean_{t \sim \uniform[0, T], \Mat{c} \sim p_c(\Mat{c}), \Mat{\epsilon} \sim \gauss(\Mat{0}, \Mat{I})} \left[ \omega(t) \frac{\partial g(\Mat{\theta}, \Mat{c})}{\partial \Mat{\theta}} \sigma_t \nabla \log p_t(\Mat{x}_t | \Mat{y})\right] \\
&= \nabla_{\Mat{\theta}} J_{Ent}(\Mat{\theta}, \lambda),
\end{align}
by merging two expectations.
\end{proof}

Further on, our CFG trick implementation of ESD (Eq. \ref{eqn:esd_cfg}) can be regarded as a corollary of Theorem \ref{thm:ent_equal_kl} and Lemma \ref{lem:grad_vsd}:
\begin{theorem}[Classifier-Free Guidance Trick] \label{lem:esd_cfg}
For any $\Mat{\theta}$ and $\lambda \in \real$, $\nabla_{\Mat{\theta}} J_{Ent}(\Mat{\theta}, \lambda)$ equals to the following:
\begin{align}
\begin{split}
-\mean_{t \sim \uniform[0, T], \Mat{c} \sim p_c(\Mat{c}), \Mat{\epsilon} \sim \gauss(\Mat{0}, \Mat{I})} \left[ \omega(t) \frac{\partial g(\Mat{\theta}, \Mat{c})}{\partial \Mat{\theta}} \left(\sigma_t \nabla \log p_t(\Mat{x}_t | \Mat{y}) - \lambda \sigma_t \nabla \log q^{\Mat{\theta}}_t(\Mat{x}_t | \Mat{y}) - (1 - \lambda) \sigma_t \nabla \log q^{\Mat{\theta}}_t(\Mat{x}_t | \Mat{c}, \Mat{y}) \right) \right]
\end{split},
\end{align}
where $\Mat{x}_t = \alpha_t \Mat{x}_0 + \sigma_t \Mat{\epsilon}$, and $\Mat{x}_0 = g(\Mat{\theta}, \Mat{c})$.
\end{theorem}
\begin{proof}
By Theorem \ref{thm:ent_equal_kl}, we know that $\nabla_{\Mat{\theta}} J_{Ent}(\Mat{\theta}, \lambda) = \lambda \nabla_{\Mat{\theta}} \overline{J_{KL}}(\Mat{\theta})  + (1-\lambda) \nabla_{\Mat{\theta}} J_{KL}(\Mat{\theta})$. Moreover, by Lemma \ref{lem:grad_vsd}, we have $\nabla_{\Mat{\theta}} J_{KL}(\Mat{\theta}) = \nabla_{\Mat{\theta}} J_{VSD}(\Mat{\theta})$.
As a result, the following can be derived:
\begin{align}
&\nabla_{\Mat{\theta}} J_{Ent}(\Mat{\theta}, \lambda) = \nabla_{\Mat{\theta}} \overline{J_{KL}}(\Mat{\theta}) + (1-\lambda) \nabla_{\Mat{\theta}} J_{VSD}(\Mat{\theta}) \\
&= -\lambda \mean_{t \sim \uniform[0, T], \Mat{c} \sim p_c(\Mat{c}), \Mat{\epsilon} \sim \gauss(\Mat{0}, \Mat{I})} \left[ \omega(t) \frac{\partial g(\Mat{\theta}, \Mat{c})}{\partial \Mat{\theta}} \left(\sigma_t \nabla \log p_t(\Mat{x}_t | \Mat{y}) - \sigma_t \nabla \log q^{\Mat{\theta}}_t(\Mat{x}_t | \Mat{y}) \right) \right] \\
&\quad -(1 - \lambda)\mean_{t \sim \uniform[0, T], \Mat{c} \sim p_c(\Mat{c}), \Mat{\epsilon} \sim \gauss(\Mat{0}, \Mat{I})} \left[ \omega(t) \frac{\partial g(\Mat{\theta}, \Mat{c})}{\partial \Mat{\theta}} \left(\sigma_t \nabla \log p_t(\Mat{x}_t | \Mat{y})  - \sigma_t \nabla \log q^{\Mat{\theta}}_t(\Mat{x}_t | \Mat{c}, \Mat{y}) \right) \right],
\end{align}
as desired after merging two expectations.
\end{proof}

\section{Illustrative Examples}
\label{sec:toy_examples}

\paragraph{Gaussian Distribution Fitting.}
In this section, we provide the necessary details on Fig. \ref{fig:gaussian_example}, where we fit a 2D Gaussian distribution via SDS, VSD, and ESD.
Suppose the targeted Gaussian distribution is $p_0(\Mat{x}_0) = \gauss(\Mat{x}_0 | \Mat{\mu}^*, \Mat{\Sigma}^*)$, where $\Mat{\mu}^* \in \real^{D}$ is the mean vector and $\Mat{\Sigma}^* \in \real^{D 
\times D}$ is the positive definite covariance matrix.
Define differentiable function $g(\{\Mat{b}, \Mat{A}\}, \Mat{c}) = \Mat{b} + \Mat{A}\Mat{c}$, where $\Mat{b}$ and $\Mat{A}$ are parameters to be fitted, and $\Mat{c} \sim \gauss(\Mat{0}, \Mat{I})$ is a random variable sampled from a standard Gaussian distribution.
It is obvious that probability density function of $g(\{\Mat{b}, \Mat{A}\}, \Mat{c})$ is $q_0^{\Mat{b}, \Mat{A}}(\Mat{x}_0) = \gauss(\Mat{x}_0 | \Mat{b}, \Mat{A}\Mat{A}^\top)$.
Our objective is to match $p_0$ and $q_0^{\Mat{b}, \Mat{A}}$ by optimizing parameters $\Mat{b}$ and $\Mat{A}$ with SDS, VSD, and ESD.
Notice that diffusion perturbed $p_0$ and $q_0^{\Mat{b}, \Mat{A}}$ are still Gaussian distributions.
And the score functions used in score distillation can be computed in the closed form:
\begin{align}
& \nabla\log p_t(\Mat{x}_t) = (\alpha_t^2 \Mat{\Sigma}^{*} + \sigma_t^2 \Mat{I})^{-1} (\alpha_t \Mat{\mu}^* - \Mat{x}_t), \\
& \nabla\log q_t^{\Mat{b}, \Mat{A}}(\Mat{x}_t) = (\alpha_t^2 \Mat{A}\Mat{A}^\top + \sigma_t^2 \Mat{I})^{-1} (\alpha_t \Mat{b} - \Mat{x}_t), \quad
\nabla\log q_t^{\Mat{b}, \Mat{A}}(\Mat{x}_t | \Mat{c}) = \frac{\alpha_t (\Mat{b} + \Mat{A}\Mat{c}) - \Mat{x}_t}{\sigma_t^2}.
\end{align}
In our experiments, we run score distillation for 2k steps with 100 warm-up steps.
The learning rate is set to 0.01.
We observe that SDS and VSD have similar convergence behavior as maximal likelihood methods.
When increasing $\lambda$ from 0 to 1, i.e., enhancing the effect of entropy maximization, the fitted distribution gradually covers the targeted support.

\paragraph{2D Image Reconstruction.}
We also conduct 2D image experiments to demonstrate the working mechanism of ESD.
We focus on the image reconstruction task via partial observations \cite{wang2024patch, daras2024ambient}, where the optimized parameters represent a high-resolution image and a random small window of the image will be rendered at each iteration of score distillation.
Formally, let $\Mat{\theta}$ denote the high-resolution image, and $g(\Mat{\theta}, \Mat{m}) = \Mat{\theta} \odot \Mat{m}$, where $\Mat{m}$ is a random binary mask.
We choose $\nabla \log p_t(\Mat{x}_t | \Mat{y})$ as a pre-trained text-to-image diffusion model, and $\Mat{y}$ is the text prompt, specified as ``An astronaut riding a horse in space'' in our experiments.
Akin to \cite{wang2023prolificdreamer}, $\nabla \log q_t(\Mat{x}_t)$ and $\nabla \log q_t(\Mat{x}_t | \Mat{m})$ are fitted via LoRA using cropped images.
During training, we fix training steps as 10k, learning rate as 1e-2 for $\Mat{\theta}$, and 1e-4 for LoRA.
We also apply the cosine learning rate scheduler with 100 warm-up steps.
We present qualitative results in Fig. \ref{fig:image_inpaint_example}. 
It demonstrates that SDS and VSD cause ``Janus''-like problems where each image contains duplicated instances while ESD avoids such issues and generates only one target object.

\begin{figure}[t]
    \centering
    \includegraphics[width=0.5\linewidth]{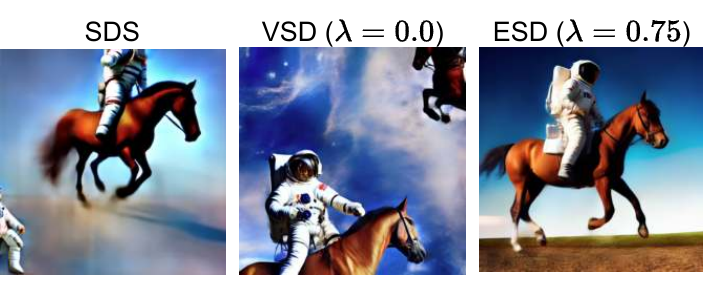}
    \caption{\small \textbf{Image Reconstruction Example.} We leverage SDS, VSD, and ESD to recover a high-resolution 2D image by matching the distribution of its random crops with a pre-trained text-conditioned diffusion model. Prompt: An astronaut riding a horse in space.}
    \label{fig:image_inpaint_example}
    \vspace{-1em}
\end{figure}

\begin{figure*}[t]
    \centering
    \includegraphics[width=\linewidth]{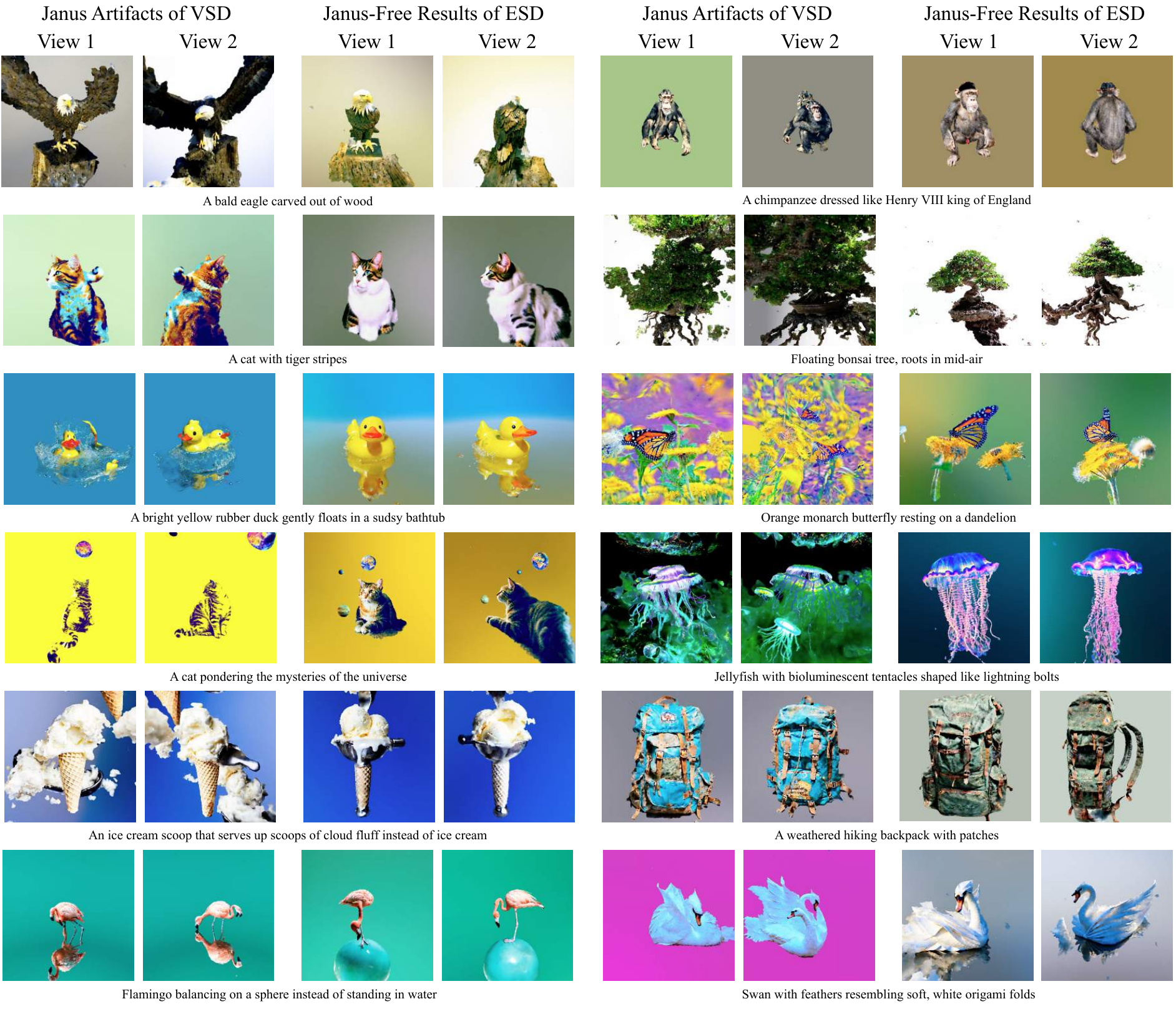}
    \vspace{-1em}
    \caption{\small \textbf{More Qualitative Results.} We present the two views of each object synthesized by VSD (ProlificDreamer) and our method, respectively. Best view in an electronic copy.}
    \label{fig:more_res}
    \vspace{-1em}
\end{figure*}

\section{Experiment Details}
\label{sec:expr_details}

In this section, we provide more details on the implementation of ESD and the compared baseline methods.
All of them are implemented under the \texttt{threestudio} framework and include three stages: coarse generation, geometry refinement, and texture refinement, following \cite{wang2023prolificdreamer}.
For the coarse generation stage, we adopt foreground-background disentangled hash-encoded NeRF \cite{muller2022instant} as the underlying 3D representation, and DMTet \cite{shen2021deep} for two refinement stages.
All scenes are trained for 25k steps for the coarse stage, 10k steps for geometry refinement, and 30k steps for texture refinement, for the sake of fair comparison.
At each iteration, we randomly render one view (i.e., batch size equals one).
We progressively adjust rendering resolution: within the first 5k steps, we render at 64$\times$64 resolution and increase to 256$\times$256 resolution afterward.

\paragraph{SDS \cite{poole2022dreamfusion}.}
Following the original paper, we set the CFG weight to 100.
Additionally, we encourage sparsity of the density field and penalize the mismatch between orientation and predicted normal maps.
Lighting augmentation is also enabled for SDS.
The geometry refinement step is directly borrowed from VSD: a DMTet is initialized by NeRF's density field via marching cube while the end-to-end optimization with SDS is then conducted on the geometry representation for both geometry and texture.

\paragraph{VSD \cite{wang2023prolificdreamer}.}
We reuse the standard setting of VSD for all three stages.
In particular, we fix the CFG coefficient to 7.5 and only use single-particle VSD, conforming with our theoretical analysis.
During the geometry refinement stage, we adopt SDS guidance instead of VSD.

\paragraph{Debiased-SDS \cite{hong2023debiasing}.}
Our implementation of Debiased-SDS is built upon SDS.
We enable both score debiasing and prompt debiasing.
For score debiasing, we follow the default setting and linearly increase the absolute threshold for gradient clipping from 0.5 to 2.0.
All other hyperparameters follow from SDS.

\paragraph{Perp-Neg \cite{armandpour2023re}.}
Perp-Neg implementation is based on SDS as well.
As suggested by the original paper, in positive prompts, we leverage weights $r_{interp} = 1 - 2|\text{azimuth}| / \pi$ for front-side prompt interpolation and $r_{interp} = 2 - 2|\text{azimuth}| / \pi$ for side-back interpolation. 
In negative prompts, the interpolating function is chosen as the shifted exponential function $\alpha \exp(-\beta r_{interp}) + \gamma$.
Specifically, we choose $\alpha_{sf} = 1, \beta_{sf} = 0.5, \gamma_{sf} = -0.606, \alpha_{fsb} = 1, \beta_{fsb} = 0.5, \gamma_{fsb} = 0.967, \alpha_{fs} = 4, \beta_{fs} = 0.5, \gamma_{fs} = -2.426, \alpha_{sf} = 4, \beta_{sf} = 0.5, \gamma_{sf} = -2.426$.
See \cite{armandpour2023re} for more details on the meaning of these hyperparameters.

\paragraph{ESD.}
Our ESD implementation is similar to VSD.
We leverage the extrinsics matrix ($4 \times 4$) as the camera pose embedding, and condition the diffusion model by replacing its class embedding branch.
We introduce CFG trick to linearly mix camera-conditioned and unconditioned score functions fine-tuned with rendered images.
We find CFG 0.5 generally yields desirable results.
We also set the probability of unconditioned training to 0.5.
In particular, view-dependent prompting is disabled for the fine-tuned score function.

\section{More Qualitative Results}
\label{sec:more_vis_res}

In this section, we present more qualitative results in Fig. \ref{fig:more_res}.
All text prompts are generated by GPT-4V and directly picked from \cite{wu2024gpt}.
We mainly compare ESD with VSD to highlight the influence of the entropy regularization term.
The observation is consistent with our main text.
The outcomes of VSD often exhibit broken geometries, duplicated objects, and multiple signature views, which contradict the inherent characteristics of the generated subjects.
ESD can effectively mitigate ``Janus'' issues and generate more realistic contents.

\section{Numerical Evaluation}
\label{sec:full_metrics}

\paragraph{Human Evaluation Criteria.}
In our human evaluation of \underline{S}uccessful generation \underline{R}ate (SR), a text prompt is labeled as ``successfully generated'' if at least one of three random seeds yields a generation satisfying the criteria: (\cite{armandpour2023re}, Appendix A.2):

\begin{enumerate}[leftmargin=2\parindent]
\item The rendered images show the requested object(s), which is positioned on the correct view.
\item The rendered images do not show hallucination including counterfactual details, for example, a panda has three ears.
\item The rendered images do not have unrealistic color or texture or massive floaters.
\end{enumerate}

\paragraph{Extension of Tab. \ref{tab:metric}.}
In Tab. \ref{tab:full_metrics}, we include standard deviations of all numerical results presented in Tab. \ref{tab:metric}.
We note that ESD exhibits a smaller variance of different metrics, indicating its training might be well-regularized and more robust.

\begin{table}[h]
\centering
\vspace{-0.5em}
\caption{\small \textbf{Quantitative Comparisons.} $(\downarrow)$ means the lower the better, and $(\uparrow)$ means the higher the better.}
\label{tab:full_metrics}
\vspace{-0.5em}
\begin{tabular}{c|cccccc}
\toprule
& CLIP ($\downarrow$) & FID ($\downarrow$) & IQ ($\downarrow$) & IV ($\uparrow$) & IG ($\uparrow$) & SR ($\uparrow$) \\
\hline
SDS \citep{poole2022dreamfusion}  & 0.737\textpm0.068 & 291.860\textpm61.242 & 4.295\textpm0.419 & \textbf{4.8552\textpm0.342} & 0.123\textpm0.053 & 15.00\% \\
VSD \citep{wang2023prolificdreamer} & 0.725\textpm0.072 & 265.141\textpm58.549 & 3.149\textpm1.234 & 3.5712\textpm1.345 & 0.137\textpm0.061 & 19.17\% \\
ESD (Ours) & \textbf{0.714\textpm0.065} & \textbf{235.915\textpm56.558} & \textbf{3.135\textpm1.088} & 4.0314\textpm1.285 & \textbf{0.327\textpm0.185} & \textbf{55.83\%} \\
\bottomrule
\end{tabular}
\vspace{-1em}
\end{table}

\paragraph{Breakdown Table.}
We provide a breakdown table to present quantitative evaluations of results in Fig. \ref{fig:res_janus}.
The numbers are reported in Tab. \ref{tab:breakdown}.
The conclusion is consistent with our argument in Sec. \ref{sec:expr}.
Our ESD consistently outperforms all the compared baseline methods, especially in FID and IG.
This implies that ESD effectively boosts the view diversity and accurately matches the distribution between pre-trained image distribution and rendered image distribution.

\begin{table}[h]
\vspace{-0.5em}
\centering
\caption{\small \textbf{Quantitative Comparisons.} $(\downarrow)$ means the lower the better, and $(\uparrow)$ means the higher the better.}
\label{tab:breakdown}
\vspace{-0.5em}
\resizebox{0.9\linewidth}{!}{
\begin{tabular}{c|ccccc|ccccc}
\toprule
& CLIP ($\downarrow$) & FID ($\downarrow$) & IQ ($\downarrow$) & IV ($\uparrow$) & IG ($\uparrow$) & CLIP ($\downarrow$) & FID ($\downarrow$) & IQ ($\downarrow$) & IV ($\uparrow$) & IG ($\uparrow$) \\
\hline
& \multicolumn{5}{c|}{\small Michelangelo style statue of dog reading
news on a cellphone} & \multicolumn{5}{c}{\small A rabbit, animated movie character, high detail 3d model} \\
\hline
SDS \cite{poole2022dreamfusion}  & 0.694 & 365.304 & 4.469 & \textbf{5.119} & 0.145
& \textbf{0.712} & 200.084 & 4.365 & \textbf{4.970} & \textbf{0.138} \\
VSD \cite{wang2023prolificdreamer} & 0.758 & 296.168 & \textbf{2.514} & 3.041 & 0.209 
& 0.720 & 150.120 & \textbf{1.083} & 1.173 & 0.083 \\
Debiased-SDS \cite{hong2023debiasing} & 0.778 & 351.493 & 4.058 & 4.814 & 0.186
& 0.735 & 216.058 & 4.443 & 4.857 & 0.093 \\
Perp-Neg \cite{armandpour2023re} & 0.793 & 306.918 & 3.970 & 4.572 & 0.151
& 0.727 & 176.279 & 2.453 & 2.665 & 0.086 \\
ESD (Ours) & \textbf{0.685} & \textbf{292.716} & 2.523 & 4.080 & \textbf{0.617}
& 0.725 & \textbf{149.763} & 1.385 & 1.567 & 0.132 \\
\hline
& \multicolumn{5}{c|}{\small A rotary telephone carved out of wood} & \multicolumn{5}{c}{\small A plush dragon toy} \\
\hline
SDS \cite{poole2022dreamfusion}  & 0.853 & 309.929 & 3.478 & 4.179 & 0.202
 & 0.889 & 243.984 & 4.622 & \textbf{5.008} & 0.084 \\
VSD \cite{wang2023prolificdreamer} & 0.855 & 305.920 & 3.469 & 4.214 & 0.214
& 0.821 & 273.495 & \textbf{4.382} & 4.728 & 0.078 \\
Debiased-SDS \cite{hong2023debiasing} & 0.927 & 313.893 & 4.098 & 4.201 & 0.025
& 0.878 & 262.474 & 4.827 & 4.954 & 0.026 \\
Perp-Neg \cite{armandpour2023re} & 0.868 & 308.554 & 3.488 & 4.021 & 0.153
& 0.839 & 309.276 & 4.691 & 4.816 & 0.027 \\
ESD (Ours)  & \textbf{0.846} & \textbf{299.578} & \textbf{3.332} & \textbf{4.439} & \textbf{0.366}
& \textbf{0.815} & \textbf{237.518} & 4.436 & 4.971 & \textbf{0.121} \\
\bottomrule
\end{tabular}}
\vspace{-1em}
\end{table}

\section{Limitations and Failure Cases}
\label{sec:fail_case}

We note that by Theorem \ref{thm:ent_equal_kl}, ESD still optimizes for a mode-seeking KL divergence.
This suggests that ESD may still lead to mode collapse especially when the target image distribution is overly concentrated on one peak \citep{salimans2016improved}.
Careful tuning of $\lambda$ is also necessary to balance the per-view sharpness/details and cross-view diversity.
It also remains open whether ESD can further benefit multi-particle VSD or amortized text-to-3D training \citep{lorraine2023att3d}.

Below we present a failure case produced by our ESD in Fig. \ref{fig:failure_case}, where the back view of the marble still contains a mouse face while the side views exhibit duplicate ears.
We point out that even though ESD can encourage diversity among views, however, it may still incline to one mode when the target image distribution is overwhelmingly concentrated at one point.
The specified text prompt in Fig. \ref{fig:failure_case} is in this case as we observe the majority of sampled images from a pre-trained diffusion model with the corresponding prompt are the frontal views of a marble mouse.

\begin{figure*}[b]
    \centering
    \vspace{-2em}
    \includegraphics[width=0.5\textwidth]{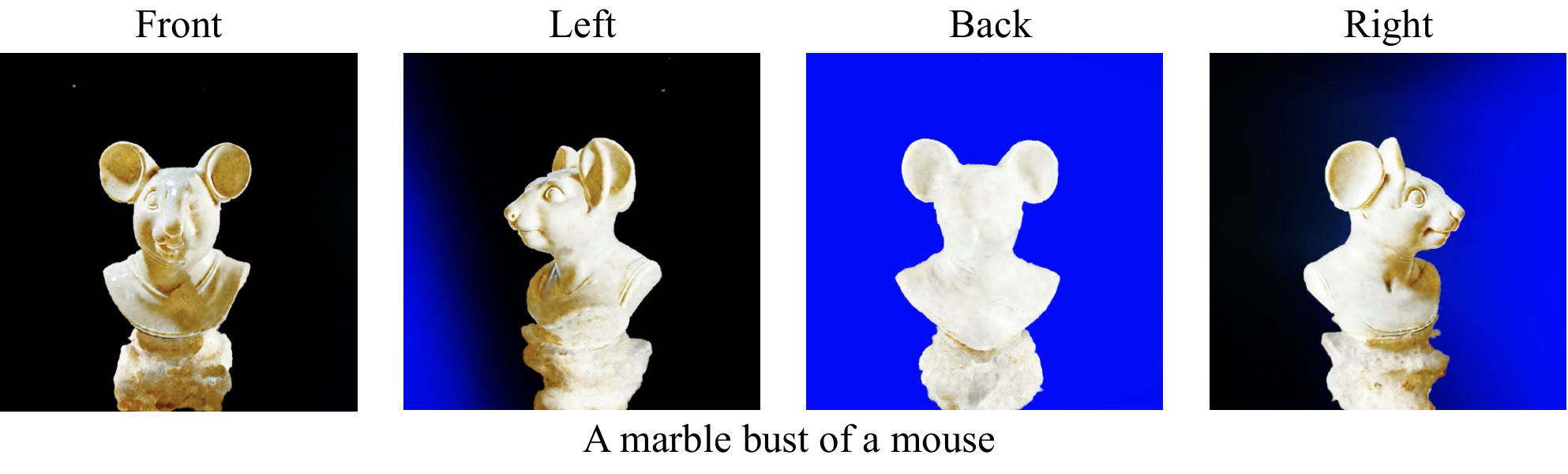}
    \vspace{-1em}
    \caption{\small \textbf{Failure case.} We present four views of a failure case yielded by ESD with prompt ``A marble bust of a mouse'' and CFG weight $\lambda = 0.5$.}
    \label{fig:failure_case}
    \vspace{-2em}
\end{figure*}

\end{document}